\pdfoutput=1
\pdftrailerid{}
\pdftrailerid{}
\documentclass{article}

\usepackage{iclr2027_conference,times}

\usepackage[utf8]{inputenc}
\usepackage[T1]{fontenc}
\usepackage{url}
\usepackage{xurl}                    %
\usepackage{booktabs}                %
\usepackage{amsfonts}
\usepackage{nicefrac}
\usepackage{microtype}
\usepackage{xcolor}
\usepackage{amsmath,amssymb,amsthm}
\usepackage{mathtools}

\usepackage{graphicx}
\usepackage{subcaption}
\usepackage{placeins}                 %

\usepackage{array}                   %
\usepackage{tabularx}                %
\usepackage{adjustbox}               %
\usepackage{multirow}                %
\usepackage{makecell}                %
\usepackage{threeparttable}          %

\usepackage{tcolorbox}
\tcbuselibrary{breakable,skins}

\usepackage{enumitem}

\usepackage{algorithm}
\usepackage{algpseudocode}           %
\floatstyle{ruled}                   %
\restylefloat{algorithm}

\usepackage{hyperref}

\usepackage[capitalize,noabbrev]{cleveref}

\usepackage{xspace}

\graphicspath{
  {figures/}
}

\hypersetup{
  colorlinks = true,
  linkcolor  = blue!70!black,
  citecolor  = green!50!black,
  urlcolor   = blue!70!black,
}

\newcommand{\secondmodel}{Qwen3-8B}

\newcommand{\E}{\mathbb{E}}
\newcommand{\norm}[1]{\left\|#1\right\|}
\newcommand{\abs}[1]{\left|#1\right|}

\newcommand{\fanout}{n}               %
\newcommand{\ndecisions}{K}           %
\newcommand{\huvar}{h_u}              %
\newcommand{\action}{a}               %
\newcommand{\traj}{\tau}              %
\newcommand{\tracepre}{\tau_{<u}}     %
\newcommand{\Rterm}{R(\traj)}         %

\newcommand{\pib}{\pi_b}             %
\newcommand{\pitheta}{\pi_\theta}    %
\newcommand{\Db}{D_b}                %

\newcommand{\bloo}[1][j]{b_{-#1}}    %
\newcommand{\Ahat}{\hat{A}}           %
\newcommand{\Astar}{A^*_{\Db}}       %
\newcommand{\QDb}{Q^{\Db}}           %
\newcommand{\VDb}{V^{\Db}}           %

\newcommand{\influence}{I(J; Y \mid \huvar)}

\DeclareMathOperator{\Var}{Var}       %

\newcommand{\cA}{\mathcal{A}}        %
\newcommand{\cL}{\mathcal{L}}        %

\newtheorem{theorem}{Theorem}

\newtheorem{proposition}[theorem]{Proposition}

\newtheorem*{remark}{Remark}
\newtheorem{definition}{Definition}

\newtcolorbox{AppBox}[1]{%
  enhanced, breakable, sharp corners,
  boxrule=0pt, leftrule=2.5pt, rightrule=0pt, toprule=0pt, bottomrule=0pt,
  colframe=black!75, colback=black!2,
  left=8pt, right=8pt, top=8pt, bottom=8pt,
  before skip=14pt, after skip=14pt,
  fonttitle=\large\bfseries, coltitle=black,
  title=#1, colbacktitle=black!2,
  attach title to upper=\vspace{6pt}\par
}
\newtcolorbox{PromptBox}[1]{%
  enhanced, breakable, arc=3pt,
  boxrule=0.5pt, colframe=black!25, colback=black!1,
  left=10pt, right=10pt, top=8pt, bottom=8pt,
  before skip=12pt, after skip=12pt,
  fonttitle=\bfseries\sffamily, coltitle=black!80,
  title=#1, colbacktitle=black!1,
  attach title to upper=\vspace{6pt}\par
}
\newtcolorbox{CaseBox}[1]{%
  enhanced, breakable, sharp corners,
  boxrule=0pt, toprule=1.5pt, bottomrule=1.5pt,
  colframe=black!80, colback=white,
  left=6pt, right=6pt, top=10pt, bottom=10pt,
  before skip=16pt, after skip=16pt,
  fonttitle=\bfseries, coltitle=black,
  title=#1, colbacktitle=white,
  attach title to upper=\vspace{8pt}\par
}
\newenvironment{AppItemize}{%
  \begin{itemize}
    \setlength{\itemsep}{4pt}
    \setlength{\parskip}{0pt}
    \setlength{\topsep}{4pt}
    \setlength{\partopsep}{0pt}
    \setlength{\parsep}{0pt}
  }{\end{itemize}}

\usepackage{pifont}

\newcolumntype{C}[1]{>{\centering\arraybackslash}p{#1}}
\newcolumntype{L}[1]{>{\raggedright\arraybackslash}p{#1}}
\newcolumntype{R}[1]{>{\raggedleft\arraybackslash}p{#1}}

\title{The Trace Is the State:\\
  Exact Credit Assignment\\
  for LLM Agent Teams}

\author{%
  Yanjun Chen$^{1,2}$\thanks{Correspondence: \texttt{yan-jun.chen@connect.polyu.hk}}
  \And
  Yirong Sun$^{1}$
  \And
  Hanlin Wang$^{2}$
  \And
  Jinghan Wang$^{3}$
  \AND
  Xinming Zhang$^{1}$
  \And
  Xiaoyu Shen$^{1}$
  \And
  Wenjie Li$^{2}$
  \And
  Wei Zhang$^{1}$\thanks{Corresponding author: \texttt{zhw@eitech.edu.cn}}
  \AND
  \mdseries
  $^{1}$Eastern Institute of Technology \quad
  $^{2}$The Hong Kong Polytechnic University \\
  \mdseries
  $^{3}$Harbin Institute of Technology
}

\iclrfinalcopy   %

\begin{document}
\maketitle
\lhead{Preprint. Under review.}

\begin{abstract}
Credit assignment for a team of LLM agents, what each message was worth, has mostly been treated as prediction: a learned critic, a trajectory-level score, or an agent ablation stands in for the counterfactual that defines credit, which is rarely run. Teams that communicate through a shared context are different: when everything a downstream agent reads is written into the trace, the trace is the state, and that counterfactual can be executed. A credit signal can then be judged as any estimator is: by bias, variance, and agreement with an independent reference. C3, credit assignment by counterfactual continuation, substitutes one message at a decision point and continues the run to the terminal reward, so its credit is unbiased, exact up to Monte Carlo error. Given the sampled alternatives, that error's variance follows a derived law with no term for the number of agents, and the observed noise follows the law on 6 workflows of 2 to 10 decision points. On the 2-agent chain, from 4 continuations, C3 ranks alternatives at 0.69 rank correlation against a 16-continuation reference, near the 0.73 at which 2 such references agree; a critic trained on the same rollouts reaches 0.29. Used as the advantage in training a 2-agent team, C3 beats MAGRPO, the stronger baseline in our comparison, and spends 37\% fewer training tokens than MAPPO, since only the messages downstream of a decision point are regenerated. When the trace is the state, credit need not be predicted; it can be exact.
\end{abstract}

\section{Introduction}
\label{sec:introduction}

A team of LLM agents cooperates by reading from and writing to a shared context~\citep{li2023camel,qian2024chatdev,wu2024autogen}: one agent drafts, another executes, a third checks, a router decides who speaks next, and the wiring may be a chain, a tree, a graph with parallel branches, or a loop that runs until some agent stops it. A run produces a long trace and one outcome, and training the team from that outcome requires answering, for every message in the trace, how much it was worth: credit assignment, the central obstacle to training cooperative LLM systems from a single terminal reward. The deeper the team, the more decisions must be priced from the same outcome.

Reinforcement learning has a standard answer: compare the action taken with the actions the agent could have taken instead, its counterfactual advantage. An LLM's action is an entire message, so alternatives cannot be enumerated, only sampled, and the question becomes what each is worth. There are two ways to find out. One can fit a model that predicts the value of an alternative, a critic~\citep{yu2022surprising,foerster2018counterfactual}, or one can carry the alternative forward: let the rest of the team run to the end and observe the reward. Most work on LLM agent teams predicts rather than runs, or sidesteps the question with a trajectory-level score shared by every decision~\citep{zhang2024cory,liu2025magrpo} or with agent ablation, comparing what remains~\citep{wan2025ccpo}. The choice is inherited from multi-agent reinforcement learning, where it is correct: in a physical environment revisiting a decision point means resetting the world, which is costly or impossible, so the outcome of a different action is estimated instead~\citep{wolpert2001optimal,foerster2018counterfactual}. The assumption traveled with the tools into LLM teams, even though single-agent LLM training already values a step by re-running from its prefix~\citep{kazemnejad2024vineppo}. For a team, two questions follow: when is a re-run from the prefix exact, and how does its credit compare, as an estimator, with predicted credit?

Agent teams that communicate through text can revisit any decision point whose inputs the trace records. When every message, tool call, memory read, and routing decision is written into the execution trace, the trace is the complete state of the system. Reset the system to any point in the trace, replace one message, re-run everything downstream: the outcome of the alternative is executed, not predicted, and what remains is sampling noise, which can be quantified. Credit is exact in this sense. The condition that makes a re-run exact, that the trace is the state, holds by construction for teams whose components read only the trace, and only unrecorded hidden state voids it (\cref{sec:sufficiency}). C3, credit assignment by counterfactual continuation, is the estimator built on this fact (\cref{fig:overview}). At each decision point it samples alternatives from the behavior policy $\pib$, runs each to the terminal reward, and compares each with the continuation-weighted mean return of the other alternatives (\cref{sec:method}).

\begin{figure}[t]
  \centering
  \includegraphics[width=\textwidth]{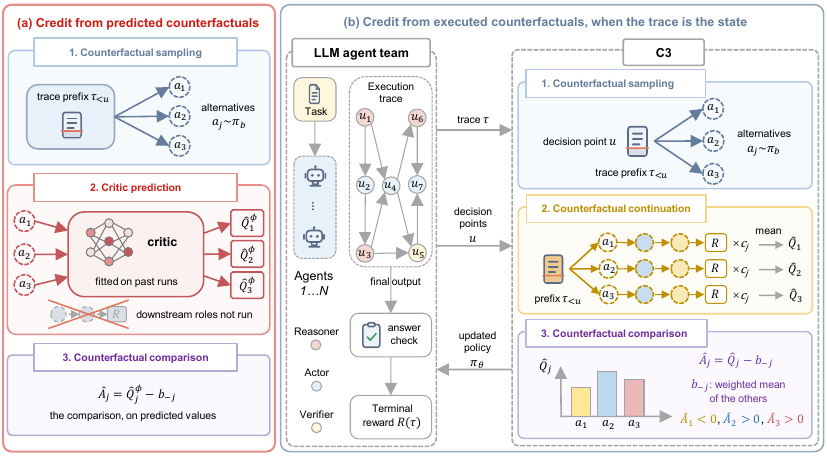}
  \caption{Two ways to value an alternative message. (a) Inherited: the decision point cannot be revisited, so a critic with parameters $\phi$ predicts each alternative's value $\hat{Q}^{\phi}_j$. (b) When every role reads and writes only the trace, C3 resets to the prefix $\tracepre$, substitutes an alternative $\action_j$, runs the downstream roles to the reward $c_j$ times, and compares the mean reward $\hat{Q}_j$ with the continuation-weighted mean $\bloo$ of the other alternatives: executed, not predicted. Rank correlation against a 16-continuation reference: critic 0.29, C3 from 4 continuations 0.69 (\cref{sec:findings}).}
  \label{fig:overview}
\end{figure}

Reinforcement learning trades the bias of a critic against the variance of Monte Carlo returns; where the counterfactual cannot be run, that trade can be checked only in aggregate and only on the actions taken, and a credit signal is judged by the team it trains. Where it can be run, the value of every alternative at one history is observable. There, every credit signal, whether a critic, a trajectory-level score, an agent ablation, or C3, is used to estimate the same advantage and can be held to the standard of any estimator: its bias, its variance and reliability~\citep{zhang2026casurvey}, and its agreement with a reference that shares none of its rollouts~\citep{zhang2026creditaudit}. Our contributions answer the three criteria in order and then put C3 to use:
\begin{enumerate}[leftmargin=*,itemsep=0pt,topsep=1pt]
  \item \emph{Bias.} When the trace is the state, C3 estimates the per-decision counterfactual advantage without bias and without learned parameters (\cref{sec:sufficiency,sec:method}). Agent ablation instead estimates a different quantity, the value of the agent's presence: it shifts the downstream observation law and moves every credit at that decision point by one offset that more rollouts do not shrink; a trajectory-level score gives every decision in a run the same number (\cref{sec:distshift,sec:method}).
  \item \emph{Variance and reliability.} Given the sampled alternatives, the variance of each advantage estimate follows a law set by that decision point's own budget and return variance, with no term for the number of agents; at the opening decision point of 6 workflows of 2 to 10 decision points the observed variance is 0.90 to 1.07 times the law, and split-half reliability on the same workflows follows the signal share of the variance (\cref{sec:scaling}).
  \item \emph{Agreement with a reference.} On held-out decision points of the 2-agent chain, from 4 continuations per alternative, C3 ranks alternatives at 0.69 rank correlation against a 16-continuation reference, near the 0.73 at which 2 such references agree; a critic trained on the same rollouts reaches 0.29 (\cref{sec:findings}).
  \item \emph{In use.} As the advantage in policy optimization on a 2-agent workflow, C3 beats MAGRPO, the stronger baseline, by \ensuremath{+8.28} points on MATH500 over 5 training seeds (Holm-corrected \mbox{$p$ < 0.001}; \citealp{holm1979simple}), and at a 2048-token generation limit spends 37\% fewer training tokens than MAPPO (\cref{sec:experiments}; \cref{tab:main}; \cref{app:extended_2a}).
\end{enumerate}

\section{Credit as an Executed Counterfactual}
\label{sec:problem}

The opening criterion is bias: this section defines the per-decision advantage every credit signal is used to estimate, shows that ablating an agent whose message is read downstream shifts the downstream observation law by an amount more rollouts do not shrink (\cref{prop:distshift}), and gives the condition under which C3 is unbiased (\cref{def:sufficiency}, \cref{prop:unbiased}).

\subsection{A Run Is a Sequence of Decision Points}
\label{sec:posg}

Several LLM agents cooperate by reading from and writing to a shared context, and one run is a sequence of decision points $u = 1, \ldots, \ndecisions$ over the trace they build together. At decision point $u$ a single role reads a context $\huvar$ and writes one message $\action_u$. The workflow assembles that context from the trace prefix $\tracepre$, the record of everything written before $u$, so $\huvar = \psi(\tracepre)$ for a deterministic assembly map $\psi$. Which role acts next, and which part of the trace it is shown, are themselves entries in the trace. Chains, trees, directed graphs with parallel branches, dynamic dispatch, and loops are instances of this object.

This is a cooperative partially observable stochastic game (POSG)~\citep{bernstein2002complexity} played in sequence: at $u$, agent $i = \mathrm{role}(u)$ draws $\action_u \sim \pitheta^i(\cdot \mid \huvar)$, and after $\ndecisions$ decision points a scalar terminal reward $\Rterm$ is scored from the final output (notation in \cref{app:notation}). A decision is a whole message rather than a token, so $\ndecisions$ is small (2 to 10 here).

\subsection{The Counterfactual and Two Ways to Evaluate It}
\label{sec:credit}

What a message was worth is a counterfactual question. Credit assignment estimates the per-decision advantage
\begin{equation}
\label{eq:advantage}
A(u) \;=\; \E_{\Db}\bigl[\Rterm \mid \huvar,\, \action_u\bigr] \;-\; \E_{\Db}\bigl[\Rterm \mid \huvar\bigr],
\end{equation}
what $\action_u$ contributed beyond what $\huvar$ already determined, with the continuation drawn under the behavior policy $\pib$ (distribution $\Db$), the snapshot of $\pitheta$ that generated the run; $\pib$ is the sampling policy in the PPO ratio of \cref{sec:optimization}, not the reference policy of the KL penalty. Alternatives to a whole message can be sampled but not enumerated, and there are two ways to evaluate one (\cref{fig:overview}): predict its value or execute it.

\label{sec:distshift}
Work on LLM agents inherited the constraint of physical environments, where only prediction is available once the moment has passed, and three approximations follow. A critic predicts the value of an alternative, and nothing in the run supplies a reference against which its error could be read~\citep{yu2022surprising,foerster2018counterfactual}. A trajectory-level score hands every decision in a run the same number, whatever each decision was worth~\citep{zhang2024cory,liu2025magrpo}. Agent ablation values an agent by deleting it, which changes what every downstream role that reads it receives~\citep{wan2025ccpo}. That change can be quantified. Write $\action_\emptyset$ for the null message that ablation substitutes at decision point $u_i$, and for a downstream decision point $j$ let $P^{j}_{\mathrm{pres}}$ and $P^{j}_{\mathrm{abl}}$ be the laws of the observation $j$ receives with agent $i$ present and ablated. Let $S_j$ collect the messages at $u_i$ that force a downstream observation the null message could never produce, and let $V_j(o) = \E[\Rterm \mid o_j = o]$ have oscillation $\mathrm{osc}(V_j)$.

\begin{proposition}[Distribution shift under ablation]
\label[proposition]{prop:distshift}
Assume (A1) the null message has probability below one under $\pib^i(\cdot \mid h_{u_i})$ and (A2) downstream contexts are assembled by substitution, so a message that enters a context verbatim can be read back from it (\cref{app:distshift}). Then for every downstream decision point $j$, the total variation distance satisfies $\mathrm{TV}(P^{j}_{\mathrm{pres}}, P^{j}_{\mathrm{abl}}) \geq \pib^i(S_j \mid h_{u_i})$, and the ablation bias $B_i^j = \E_{P^{j}_{\mathrm{pres}}}[V_j] - \E_{P^{j}_{\mathrm{abl}}}[V_j]$ satisfies $\abs{B_i^j} \leq \mathrm{osc}(V_j)\cdot \mathrm{TV}(P^{j}_{\mathrm{pres}}, P^{j}_{\mathrm{abl}})$; both are fixed by the two laws and $V_j$, so drawing more rollouts does not reduce them.
\end{proposition}

\noindent The proof is in \cref{app:distshift}. An agent whose message reaches no downstream observation leaves the two laws equal, so ablation is unbiased for an agent that changes nothing; where the message is read verbatim, as here, the shift is at least $1 - \pib^i(\action_\emptyset \mid h_{u_i})$. $B_i^j$ is the same for every alternative at $u_i$: it prices the agent's presence rather than the sampled message, so it leaves the order of the alternatives at $u_i$ untouched and moves every credit there by one offset; offsets differ across decision points, which matters when credit is compared across agents or decision points.

\subsection{The Trace Is the State}
\label{sec:sufficiency}

Agents that communicate through text can be re-run from any decision point whose inputs the trace records, and the re-run is exact when the trace is the whole state. Let $S_u$ be the complete state at $u$, everything that bears on the law of what follows, and let $\tracepre$ be the trace prefix the system records.

\begin{definition}[The trace is the state]
\label{def:sufficiency}
The trace is the state at decision point $u$ if, for every message $\action$ to be evaluated,
\begin{equation}
\label{eq:sufficiency}
\cL\bigl(Y_{>u},\, R \;\big|\; \mathrm{do}(\action_u = \action),\, \tracepre\bigr) \;=\; \cL\bigl(Y_{>u},\, R \;\big|\; \mathrm{do}(\action_u = \action),\, S_u\bigr),
\end{equation}
where $Y_{>u}$ collects the downstream messages and $\cL(\cdot)$ denotes a law.
\end{definition}

Under \eqref{eq:sufficiency}, resetting to $\tracepre$ and injecting $\action$ yields a draw from the same interventional law as resetting the true state would, so \eqref{eq:advantage} is identified, under the behavior policy's continuation, with no fitted model in the path. That identification is what \emph{exact} means here: the counterfactual is executed and observed, so credit computed from it is exact up to Monte Carlo error, whose variance, given the sampled alternatives, \cref{sec:budget} derives. Only the prefix is held fixed; downstream contexts diverge freely after different messages, since that divergence is the effect being priced.

In the workflows of this paper the condition holds by construction, because every role's input is assembled by substitution from the task text, the messages upstream of it, and fixed templates, all recorded in the trace. Elsewhere it is a recording requirement that sorts systems into three cases. \emph{Recorded}: every causally relevant variable, tool returns and memory contents and routing decisions included, is written into the trace and restored on reset, so \eqref{eq:sufficiency} holds. \emph{Resampled}: a genuinely random component is drawn afresh from its law given the recorded prefix, and the estimate stays unbiased at a variance cost. \emph{Hidden state}: something bearing on the continuation is neither recorded nor resampled, as when a decision point depends on an external world that text cannot restore~\citep{gonuguntla2026replaygap}, and no guarantee applies. A system that cannot be shown to be recorded or resampled is treated as hidden state.

\subsection{The Estimator}
\label{sec:method}

C3 turns \eqref{eq:advantage} into three operations, none of which adds a learned parameter (\cref{fig:overview}b).

\textbf{Counterfactual sampling.}
\label{sec:action_sampling}
At each iteration a snapshot $\pib \leftarrow \pitheta$ of the current policy is taken and held fixed for the iteration, and base runs are executed under it, recording at every decision point the acting role, the context $\huvar$, and the trace prefix $\tracepre$. For each recorded context, $\fanout \geq 2$ alternatives $\action_j \sim \pib(\cdot \mid \huvar)$ are drawn, forming a group of alternatives that share one context exactly; only the message under evaluation differs across the group, so the shift of \cref{prop:distshift} never enters.

\textbf{Counterfactual continuation.}
\label{sec:rollout}
The system is reset to $\tracepre$, $\action_j$ is injected in place of the recorded message, and the downstream roles finish the run under $\pib$. Each alternative receives $c_j \geq 1$ independent rollouts, giving
\begin{equation}
\label{eq:qdb}
\hat{Q}^{\Db}(\huvar, \action_j) \;=\; \frac{1}{c_j} \sum_{m=1}^{c_j} R_m,
\end{equation}
an unbiased estimate of $\QDb(\huvar, \action_j) = \E_{\Db}[\Rterm \mid \huvar, \action_j]$, under a per-instance budget that fixes the total rollouts at $B$ across decision points and alternatives. Because the prefix is reused rather than regenerated, an evaluation costs only the workflow downstream of $u$.

\textbf{Counterfactual comparison.}
\label{sec:loo}
C3 compares an alternative against its peers: its baseline is the mean return of the other alternatives at the same decision point, weighted by their rollout counts (with equal counts the leave-one-out baseline of RLOO~\citep{kool2019buy,ahmadian2024back}, related to difference rewards~\citep{wolpert2001optimal}); the baseline is subtracted, and the advantage enters the clipped objective below unchanged in sign and order (\cref{app:hyperparams}). A within-group baseline cancels the between-group variation that a global baseline would carry in, since different groups face different task instances~\citep{greensmith2004variance}. For $\action_j$ in a group of $\fanout$ alternatives with rollout counts $\{c_k\}$ and $C = \sum_k c_k$,
\begin{equation}
\label{eq:bloo}
\bloo \;=\; \frac{1}{C - c_j} \sum_{k \neq j} c_k\, \hat{Q}^{\Db}(\huvar, \action_k),
\qquad
\Ahat_j \;=\; \hat{Q}^{\Db}(\huvar, \action_j) \;-\; \bloo .
\end{equation}
\begin{proposition}[Unbiasedness]
\label[proposition]{prop:unbiased}
When the trace is the state at $u$ (\cref{sec:sufficiency}), with alternatives drawn independently from $\pib$, rollouts independent across alternatives under $\Db$, and non-adaptive rollout counts, $\Ahat_j$ is, conditional on $\action_j$, an unbiased estimate of $\Astar(\huvar, \action_j) = \QDb(\huvar, \action_j) - \VDb(\huvar)$, the expectation being over the rollouts and the other alternatives at the recorded prefix $\tracepre$, of which $\huvar$ is a function (proof in \cref{app:prop1}).
\end{proposition}

\noindent Excluding $\action_j$ from its own baseline stops an alternative from raising the bar it is judged against, and the context-level term cancels exactly (\cref{app:crn}).

\label{sec:optimization}
The advantages enter the clipped surrogate objective of PPO~\citep{schulman2017proximal} with importance ratio $r(\theta) = \pitheta(a \mid h) / \pib(a \mid h)$; any policy gradient method that accepts per-sample advantages can take its place. \Cref{alg:c3} lists the loop: at every decision point $\fanout$ alternatives are sampled, each is continued $c_j$ times and compared within its group, and the policy is updated on every (context, alternative, advantage) triple.

\begin{algorithm}[t]
\caption{C3.}
\label{alg:c3}
\begin{algorithmic}[1]
\Require Policy $\pitheta$, branching factor $\fanout$, budget $B$ rollouts per instance, clip threshold $\epsilon$
\For{each training iteration}
    \State $\pib \leftarrow \pitheta$ \Comment{Snapshot, held fixed for the iteration}
    \State Execute base runs under $\pib$; record the trace prefixes $\{\tracepre\}$
    \For{each decision point $(i, \huvar, \tracepre)$}
        \State Sample $\fanout$ actions: $\action_j \sim \pib^i(\cdot \mid \huvar)$ for $j = 1, \ldots, \fanout$
    \EndFor
    \For{each action $\action_j$ in each group of alternatives} \Comment{$B$ rollouts per instance; nested groups (\cref{sec:experiments}) share them}
        \State Reset the system to $\tracepre$; inject $\action_j$; roll out under $\pib$ to terminal
        \State Record $R_m$ for $c_j$ independent rollouts; compute $\hat{Q}^{\Db}(\huvar, \action_j)$ via \cref{eq:qdb}
    \EndFor
    \For{each group of alternatives}
        \State Compute $\bloo$ via \cref{eq:bloo} for all $j$
    \EndFor
    \State Update $\pitheta$ by maximizing the clipped PPO objective $\cL(\theta)$ over all $(\huvar, \action_j, \Ahat_j)$ triples
\EndFor
\end{algorithmic}
\end{algorithm}

\section{Variance and Reliability from 2 to 10 Agents}
\label{sec:diagnostics}

The second criterion is variance and reliability: how much an advantage estimate moves when its rollouts are drawn again, and so how many rollouts one decision point needs. Two quantities, computed from the groups that training already collects and needing no per-decision label, answer this, and a third says whether there is anything to credit.

\label{sec:fidelity}
\textbf{Reliability} is whether a group orders its own alternatives the same way twice: split its rollouts into 2 halves, the even-numbered and the odd-numbered continuations, build an advantage vector from each half by \eqref{eq:bloo}, and take the Spearman rank correlation between the 2 vectors, averaged over the groups whose halves both order the alternatives (\cref{app:deep_protocols}).
\textbf{Within-group variance}\label{sec:wgv}, the sample variance of the advantages across a group, is their total spread, signal and noise together (\cref{app:crn}).
\textbf{Influence}\label{sec:influence}, the mutual information $\influence$ between the injected alternative $J$ and the downstream response $Y$ (\citealp{jaques2019social}; \cref{app:influence_derivation,app:influence_dist}), says whether there is any coupling for credit to describe: near zero the downstream role is ignoring the upstream message~\citep{zhang2025lazy}.

This section and the next use Qwen3-4B-Instruct-2507~\citep{yang2025qwen3}, written Qwen3-4B below, and \secondmodel{} where named, held fixed at their released weights, so that what they report describes the estimator rather than a training outcome. The 2-agent workflow is the Reasoner $\to$ Actor pair of MARFT~\citep{liao2025marft}: the Reasoner writes a plan, the Actor writes the solution that gets graded; the 3-agent workflow adds a Verifier that checks the answer, and deeper workflows reach a 10-agent chain (\cref{app:deep_protocols}).

\label{sec:budget}
The noise on each advantage is set by the rollouts one decision point receives and by how much the returns after it vary, not by the number of agents. Let $B_u$ be the budget spent at a decision point, spread over $\fanout$ alternatives with $c = B_u / \fanout$ rollouts each. Conditional on the sampled alternatives, the rollout noise on one advantage estimate has variance $\sigma^2 \fanout^2 / (B_u(\fanout - 1))$, with $\sigma^2$ the variance of the return given the alternative, taken equal across alternatives (\cref{app:crn,app:budget_resolution}){}, and the observed noise matches this expression to within 9\% across a $3 \times 3$ grid of branching factors and continuation counts (\cref{app:budget_resolution,tab:noise_law}). A fixed per-instance budget $B$ spread evenly over $\ndecisions$ decision points gives $B_u = B / \ndecisions$, so depth enters through that division: at one rollout per alternative the branching factor is $\fanout = B / \ndecisions$, and below $\fanout = 2$ there is no other alternative to compare against, which $B = 8$ reaches at 5 decision points.
\begin{figure}[t]
  \centering
  \includegraphics[width=\textwidth]{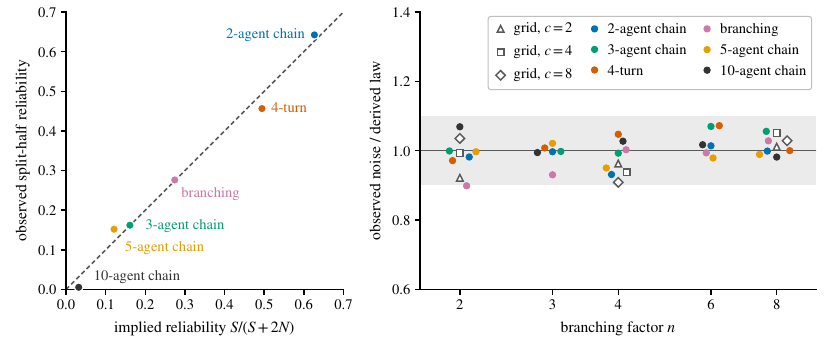}
  \caption{Reliability follows the signal share of the variance, and the noise follows the law. Left: split-half reliability of 6 workflows on a policy held fixed, against the ratio $S/(S+2N)$ of the signal variance $S$ to the variance $S+2N$ of a 2-continuation half, with $S$ and $N$ from \cref{tab:noise_law} and the reliability from \cref{tab:scaling}; the dashed line is equality. Right: the ratio of the observed noise on an advantage to the derived law, over a grid of branching factors and continuation counts on the 2-agent chain (markers) and over the 6 workflows (colors); the band marks $\pm 10\%$.}
  \label{fig:prediction}
\end{figure}
\label{sec:scaling}
Across 6 workflows of 2 to 10 decision points and every branching factor tested, the ratio of the observed noise on an advantage to the derived law stays between 0.90 and 1.07 on Qwen3-4B (\cref{tab:noise_law}), with no trend in the branching factor (rank correlation of the ratio with $\fanout$ over 43 configurations, 30 on Qwen3-4B and 13 on \secondmodel{}: \ensuremath{+0.11}, 95\% interval \ensuremath{-0.20} to \ensuremath{+0.39}), so the noise is set by the per-decision budget and the return variance.

\Cref{tab:scaling} reports split-half reliability, at the opening decision point of each workflow and on a policy held fixed, for 6 workflows of 2 to 10 decision points at a constant per-decision budget of 4 alternatives and 4 continuations each (\cref{app:deep_protocols}): from 0.64 on the 2-agent chain and 0.46 on the 4-turn workflow to 0.15 on a 5-agent chain and 0.01 on a 10-agent chain, with a 4-role workflow with 2 parallel branches at 0.28. Here these values agree with what the signal and the noise imply. Each half of a split averages 2 continuations, so its noise variance is $2N$ for $N$ the sampling variance of a 4-continuation mean, and under an additive model the correlation between the 2 halves is $S/(S+2N)$, with $S$ the variance of the alternatives' true values at a decision point, the spread the ranking rests on (\cref{app:budget_resolution}). With the variances of \cref{tab:noise_law}, this ratio matches the observed reliability of all 6 workflows to within 0.04 (\cref{fig:prediction}, \cref{tab:scaling}). What separates the workflows is therefore the signal, not the noise: $S$ falls from 0.077 on the 2-agent chain to 0.003 on the 10-agent chain, whose opening decision point, the Reader, leaves little to rank, while $N$ rises from 0.023 to 0.049, and the noise on each advantage follows the law throughout. When an even split of 8 over the decision points sets the alternatives per decision point, 4 on the 2-agent chain and 3 on the 3-agent chain, reliability is 0.23 on the 3-agent chain against 0.64 on the 2-agent chain, and \secondmodel{} reproduces the pattern at 0.34 against 0.71; \cref{fig:reliability} in \cref{app:deep_protocols} sweeps the branching factor for every workflow.

\begin{table}[t]
\caption{Constant per-decision budget: 4 alternatives per decision point, 4 continuations each, 500 groups per entry. Observed: split-half reliability on a policy held fixed, the rank agreement between the even-numbered and the odd-numbered continuations of each group. Implied: $S/(S+2N)$, the share of a 2-continuation half's variance that is signal, from the signal variance $S$ and the noise variance $N$ of \cref{tab:noise_law} (\cref{app:budget_resolution}). Groups: those whose halves both order the alternatives, out of 500 (\cref{app:deep_protocols}). The branching workflow, the one here that is not a chain, is a planner, 2 parallel solvers and an integrator.}
\label{tab:scaling}
\centering
\footnotesize
\begin{tabular}{@{}lccccc@{}}
\toprule
Workflow ($\ndecisions$) & Alternatives, $4\ndecisions$ & Observed & Implied & Within-group variance & Groups \\
\midrule
2-agent chain (2) & 8 & 0.64 & 0.63 & 0.179 & 326 \\
3-agent chain (3) & 12 & 0.16 & 0.16 & 0.092 & 333 \\
4-turn workflow (4) & 16 & 0.46 & 0.49 & 0.082 & 163 \\
Branching workflow (4) & 16 & 0.28 & 0.27 & 0.103 & 316 \\
5-agent chain (5) & 20 & 0.15 & 0.12 & 0.092 & 328 \\
10-agent chain (10) & 40 & 0.01 & 0.03 & 0.093 & 346 \\
\bottomrule
\end{tabular}
\end{table}

\section{Prediction and Ablation against the Executed Counterfactual}
\label{sec:findings}
\label{sec:reveals}

The third criterion is agreement with a reference that shares none of the estimator's rollouts, and fresh continuations of the same decision points supply one; executed returns also give the realized bias of ablation.

\label{sec:other_routes}

\textbf{From 4 continuations, C3 ranks alternatives at 0.69 rank correlation against a 16-continuation reference; a critic trained on the same rollouts reaches 0.29.} A state-value critic $V(\huvar)$ shifts every alternative at a decision point by the same constant and leaves their order untouched, so the comparison that matters is a decision-value critic $\hat{Q}(\huvar, \action)$ that scores each alternative on its own, trained on the rollouts C3 collected on the 2-agent workflow of \cref{sec:scaling}, with held-out points split from them by decision point (\cref{app:learned_critic}). Both are scored on held-out decision points against a Monte Carlo reference of 16 fresh continuations per alternative, which is unbiased under the condition of \cref{sec:sufficiency}; C3 uses its own 4 continuations per alternative, and the gap is \ensuremath{+0.40} in rank correlation (paired, $p$ < 0.001). At the reference's own budget, 16 fresh continuations per alternative agree with the reference at 0.73, so the 4-continuation estimate already sits near the level at which 2 such references agree (on the points that both references and C3 rank at every budget, 0.70 against 0.76; \cref{app:learned_critic}). At the scored points the critic spends no rollouts, the saving for which a critic accepts its bias, and retraining it on 2 and 4 times the rollouts does not close the gap (\cref{app:learned_critic}).

\textbf{Ablating an agent shifts what the roles downstream of it read.} Replacing one role's message with a null action in 150 groups of the 3-agent workflow, one per MATH500 question, and comparing the return against the group's mean under $\pib$ estimates how much of that shift reaches the reward; the per-point values, stratified by answer-level influence (\cref{app:influence_dist}) and by whether the sampled alternatives received different values, are summarized in \cref{app:ablation_bias}. The empty-message form of the null action is the agent-ablation counterfactual that CCPO~\citep{wan2025ccpo} subtracts as its baseline.

\section{C3 as the Advantage in Training}
\label{sec:experiments}
\label{sec:setup}
\label{sec:main_results}

\textbf{C3 is put to use as the advantage in policy optimization, and the 3 methods share everything but how rollouts are arranged and credited.} The tasks are math benchmarks whose final answer is checked against a reference: MATH500~\citep{hendrycks2021measuring}, AIME~\citep{aime2025}, CMATH~\citep{wei2023cmath} and GSM8K~\citep{cobbe2021training} in the main setting, and MATH500, Minerva-Math~\citep{lewkowycz2022solving}, AMC23~\citep{amc2023} and AIME with its 2024 and 2025 sets merged in the long-generation setting. The 3 methods share the 2-agent workflow of \cref{sec:diagnostics}, the released Qwen3-4B as starting point (the starting model row of \cref{tab:saturation_controls}), and the rollout budget, and differ in where the advantage comes from: a centralized critic's value per decision point (MAPPO~\citep{yu2022surprising}), one trajectory-level score normalized within a group of trajectories and handed to every decision point in them (MAGRPO~\citep{liu2025magrpo}), or the estimator of \cref{sec:loo} (C3). The behavior policy $\pib$, held fixed for the iteration, generates all rollouts at $B = 8$ rollouts per instance, which C3 splits into 2 Reasoner alternatives, each continued by 4 Actor alternatives, where the baselines draw 8 independent runs. Hyperparameters are in \cref{app:hyperparams}.

\textbf{The comparison is run under 2 settings.} \Cref{tab:main} reports the main setting: 5 training seeds per method, greedy decoding under a 512-token generation limit on MATH500, AIME 2025, CMATH and GSM8K; \cref{app:significance} tests its paired differences, including the MATH500 gain of \ensuremath{+8.28} points over MAGRPO, the stronger baseline. In the long-generation setting, one model per method is trained under a 2048-token limit and evaluated 5 times with sampled decoding on MATH500, Minerva-Math, AMC23 and AIME (\cref{tab:saturation_controls}, with the training tokens of each method). In this setting MAGRPO diverges late in training, its KL to the reference policy rising sharply and its responses running to the 2048-token cap; \cref{tab:saturation_controls} reports both its final model and its model at step 250, before the divergence. On the 4-benchmark average C3 gains \ensuremath{+2.26} points over the starting model against \ensuremath{+1.17} for MAPPO, on AMC23 it gains \ensuremath{+5.28} points, \ensuremath{+2.78} more than MAPPO, and it spends 37\% fewer training tokens than MAPPO.

\begin{table}[t]
\caption{Main setting, the 2-agent workflow on Qwen3-4B: greedy accuracy (\%) under a 512-token generation limit, mean $\pm$ std over 5 training seeds; best value per column in bold, ties included. $\dagger$Mean of MATH500, CMATH and GSM8K; AIME 2025 (30 problems) is not included in the average.}
\label{tab:main}
\centering
\footnotesize
\setlength{\tabcolsep}{3.5pt}
\begin{tabular}{@{}lccccc@{}}
\toprule
Method & MATH500 & AIME 2025 & CMATH & GSM8K & Avg.$^{\dagger}$ \\
\midrule
MAPPO & 69.3\,$\pm$\,0.9 & 3.3\,$\pm$\,0.0 & 95.3\,$\pm$\,0.2 & 92.9\,$\pm$\,0.1 & 85.8\,$\pm$\,0.3 \\
MAGRPO & 74.5\,$\pm$\,0.4 & 5.3\,$\pm$\,1.8 & 96.1\,$\pm$\,0.2 & \textbf{93.4\,$\pm$\,0.3} & 88.0\,$\pm$\,0.1 \\
C3 & \textbf{82.8\,$\pm$\,0.6} & \textbf{8.0\,$\pm$\,1.8} & \textbf{96.3\,$\pm$\,0.2} & \textbf{93.4\,$\pm$\,0.2} & \textbf{90.9\,$\pm$\,0.3} \\
\bottomrule
\end{tabular}
\end{table}

\section{What Exactness Requires}
\label{sec:failures}

\textbf{What a system has to record.} \Cref{sec:sufficiency} reduces the condition that the trace is the state, on which the unbiasedness of C3 rests, to a list of what a trace must contain. Tool and retrieval calls log their inputs, outputs, versions, and side effects, or are resampled from their law given the recorded prefix at the variance cost of \cref{sec:budget}; persistent memory is serialized into the trace rather than kept beside it; orchestration writes its routing decisions into the trace as text. The guarantee requires all three.

\textbf{Cost.} The price of exactness is the continuations, $\fanout$ alternatives with $c$ rollouts each at every credited decision point, each rollout a run of the downstream roles only. That is why C3's training tokens come in below MAPPO's in \cref{sec:experiments} although it executes every counterfactual it credits: only the messages downstream of a decision point are regenerated, where a full-episode baseline regenerates every message of every rollout; on the 2-agent workflow the $2 \times 4$ tree writes 10 messages per instance where 8 independent runs write 16.

\textbf{Tasks and rewards.}
\label{sec:discussion}
C3 is exact for the terminal reward it is given. Every task here ends in a checkable final answer, so a return varies only through the continuations, which is the noise the law of \cref{sec:scaling} describes; a reward scored by a judge adds the judge's variance to $\sigma^2$, and the same law covers it.

\section{Related Work}
\label{sec:related}

\textbf{Credit assignment in cooperative multi-agent RL.}
This literature developed under the premise that executing a counterfactual requires privileged access to the environment. MAPPO~\citep{yu2022surprising} centralizes a critic, value factorization~\citep{sunehag2017value,rashid2020monotonic,wang2020qplex} decomposes joint values under Markov and discrete-action assumptions, difference rewards~\citep{wolpert2001optimal} replace one agent's action with a fixed default, and COMA~\citep{foerster2018counterfactual} marginalizes over enumerable discrete actions at a single timestep. Where the trace is the state, that premise fails: the counterfactual COMA reaches for can be executed, downstream continuation included.

\textbf{Credit assignment for multi-agent LLM systems.}
MARFT~\citep{liao2025marft} brings reinforcement fine-tuning to LLM agent systems. MAGRPO~\citep{liu2025magrpo} and CORY~\citep{zhang2024cory} assign credit at the trajectory level, and Dr.\,MAS~\citep{feng2026drmas} traces the instability of such group-normalized objectives to a baseline shared across agents whose reward distributions differ. SHARP~\citep{li2026sharp} attributes credit by Shapley values, MAPPA~\citep{li2026mappa} by per-action process rewards from AI feedback, and CCPO~\citep{wan2025ccpo} by ablating an agent, which introduces the shift of \cref{sec:distshift}. CCPO is defined for a 2-agent chain, re-prompting the single downstream agent without the ablated message; \citet{yang2025whogets} propose, as a framework, ablating a whole agent from the recorded prefix onward, with judge scores spreading its Shapley share over its messages. \cref{app:ablation_bias} prices the shift. COSAC~\citep{deshmukh2026cosac} fits an additive per-agent decomposition without new reward calls. MALT~\citep{motwani2024malt} labels each role's outputs from a tree of sampled continuations and trains on the labels off-policy; AT-GRPO~\citep{pham2026stronger} forms groups of same-context samples for one agent at one turn, as C3 does, but scores each by an immediate verifier reward and continues only the best, where C3 continues every alternative to the terminal reward. Concurrent work budgets executed counterfactual evaluations against learned signals~\citep{huang2026cofi}, allocates credit from actions down to tokens while conserving the total~\citep{ma2026factor}, and prices what individual messages contribute by re-running a trajectory without one of its messages~\citep{yang2026wrongbutuseful}. \cref{app:mechanisms} tabulates what each of these signals restores, re-runs and compares. Relative to these, this paper holds a team's credit to the standard of an estimator: it states when re-running from the prefix is unbiased, derives how the rollout noise depends on the per-decision budget, scores learned and ablation-based credit against executed counterfactuals, and uses the continuations as on-policy advantages.

\textbf{Counterfactuals from the same prefix, and whether restoring it is exact.}
At the turn level of single-agent RL, estimating counterfactuals from the same prefix is standard~\citep{kazemnejad2024vineppo,ji2025treegrpo,liao2026tcpo,meng2026pccc,li2026cso,shah2026car,tan2026hcapo}; a report on tool-using agents instead spreads a discounted terminal return over the tool-call chunks of one trajectory~\citep{wang2025rome}. Process reward models~\citep{lightman2024lets} and turn-level reward designs~\citep{wei2025reinforcing} supply dense supervision; C3 needs only a terminal reward. Whether restoring a text history restores the decision point has been examined directly: a survey~\citep{zhang2026casurvey} discusses unbiasedness for the same-prefix estimator when the pre-intervention state is restored and both actions are supported, non-identifiability for text matching without restoration, and asks continuation-based estimators to report a test-retest floor (\cref{sec:budget}). \citet{matteson2026refeeding} shows that re-feeding a prefix on batch-variant inference kernels can differ from resuming the live decoder state by a near mean-zero perturbation, so our split-half reliability is defined on returns (\cref{sec:scaling}); \citet{gonuguntla2026replaygap} shows that restoring text does not restore an external environment, the hidden-state case of \cref{sec:sufficiency}. \citet{zhang2026creditaudit} notes that a criterion computed against the auditing tool's own estimates certifies only self-agreement; \cref{sec:findings} therefore scores C3 and the critic against 16 fresh continuations drawn under a separate seed, independent of their own rollouts.

\section{Conclusion}
\label{sec:conclusion}

Credit assignment for LLM agent teams has mostly been treated as prediction because in physical environments counterfactuals cannot be run. For agents that communicate through a shared context and read only what the trace records, the trace is the state, so they can: reset, replace one message, re-run. C3 turns this into a per-decision estimator that is unbiased, whose rollout noise follows the derived law from 2 to 10 agents, and whose agreement with an independent reference sits near the level at which two such references agree; used in training, it produces a better team and spends fewer training tokens. The prescription: log everything the roles read, spend rollouts per decision point, and check every credit signal against the counterfactual it replaces. The three criteria are not specific to C3: a critic, a trajectory-level score and an agent ablation each stand in for the same per-decision advantage. Where the trace is not the state, \cref{sec:failures} lists what to record; where it is, the executed counterfactual is the reference to score them against before any is trusted to train a team. When the trace is the state, credit need not be predicted; it can be exact.

\bibliographystyle{iclr2027_conference}
\bibliography{refs}

\appendix
\crefalias{section}{appendix}
\crefalias{subsection}{appendix}

\section{Theoretical Proofs and Derivations}
\label{app:theory}

When the trace is the state (\cref{sec:sufficiency}), counterfactual credit assignment is executed rather than predicted; the same property makes credit quality checkable against executed counterfactuals (\cref{sec:findings}). This appendix provides the complete proofs underlying both consequences. \Cref{app:distshift} shows that agent ablation induces a systematic shift in the downstream observation law that cannot be reduced by additional sampling; \cref{app:prop1} establishes the unbiasedness of the C3 advantage estimator for counterfactual rollouts from the same prefix; \cref{app:crn} derives the variance reduction properties of within-group comparison; \cref{app:influence_derivation} develops the conditional mutual information estimator for the dependence between an injected action and the downstream response; \cref{app:budget_resolution} collects the budget arithmetic behind \cref{sec:budget}; \cref{app:influence_dist} defines the response-level mutual information of \cref{sec:influence} and the answer-level quantity used in \cref{sec:reveals} and \cref{app:ablation_bias}.

\subsection{Distribution Shift Induced by Agent Ablation}
\label{app:distshift}

This subsection restates \cref{prop:distshift} from \cref{sec:distshift} in quantitative form and gives its proof. The statement separates two quantities that the qualitative reading runs together: how far ablating an agent moves the downstream observation law, and how much reward bias that movement produces. The separation is what makes the second quantity an empirical question, answered in \cref{app:ablation_bias}.

\paragraph{Setup.}
Agent $i$ acts at decision point $u_i$ with history $h_{u_i}$, and $j$ is a downstream decision point with $u_j > u_i$. Ablation replaces the action at $u_i$ by the null action $\action_\emptyset$. Write $O_j$ for the serialized observation that decision point $j$ receives, and let
\[
P^{j}_{\mathrm{pres}} \;=\; \cL\bigl(O_j \mid \action_{u_i} \sim \pib^i(\cdot \mid h_{u_i})\bigr),
\qquad
P^{j}_{\mathrm{abl}} \;=\; \cL\bigl(O_j \mid \action_{u_i} = \action_\emptyset\bigr)
\]
be its law with the agent present and with the agent ablated. Let $V_j(o) = \E[\Rterm \mid O_j = o]$ be the continuation value as a function of that observation, with oscillation $\mathrm{osc}(V_j) = \sup V_j - \inf V_j$. Total variation distance is $\mathrm{TV}(P, Q) = \sup_A |P(A) - Q(A)| = \tfrac{1}{2}\sum_o |P(o) - Q(o)|$ on the discrete observation space at hand. Two standing assumptions carry over from \cref{prop:distshift}:

\begin{enumerate}[label=(A\arabic*)]
\item \label{ass:nondegen} The behavior policy is non-degenerate at $u_i$: $\pib^i(\action_\emptyset \mid h_{u_i}) < 1$.
\item \label{ass:serial} Serialization: the observation $o_j$ at a downstream decision point is assembled by substitution from the task text and the messages upstream of it, so that when the message produced at $u_i$ enters $o_j$ verbatim there is a map $\varphi$ with $\varphi(o_j) = \action_{u_i}$.
\end{enumerate}

\noindent Let $P(o_j \mid \action)$ denote the law of $O_j$ when the action at $u_i$ is $\action$ and everything downstream is sampled as usual. Define
\[
S_j \;=\; \bigl\{\, \action : \operatorname{supp} P(o_j \mid \action) \cap \operatorname{supp} P(o_j \mid \action_\emptyset) = \emptyset \,\bigr\},
\]
the actions at $u_i$ whose downstream observations can never coincide with one produced under the null action. The definition is marginal: no conditioning on the downstream sampling stream is involved, and the bound below is a statement about the two marginal laws.

\begin{proposition}[Distribution Shift Induced by Agent Ablation, quantitative form]
\label[proposition]{prop:distshift_quant}
Under \ref{ass:nondegen} and \ref{ass:serial}, for every downstream decision point $j$ with $u_j > u_i$:
\begin{enumerate}[label=(\roman*)]
\item $\mathrm{TV}\bigl(P^{j}_{\mathrm{pres}}, P^{j}_{\mathrm{abl}}\bigr) \;\geq\; \pib^i(S_j \mid h_{u_i})$.
      In workflows where the message produced at $u_i$ appears verbatim in the prompt read at $j$, $S_j \supseteq \cA \setminus \{\action_\emptyset\}$ and therefore
      $\mathrm{TV}\bigl(P^{j}_{\mathrm{pres}}, P^{j}_{\mathrm{abl}}\bigr) \geq 1 - \pib^i(\action_\emptyset \mid h_{u_i}) > 0$.
      If $S_j = \emptyset$, this lower bound is vacuous.
\item The ablation bias
      $B_i^j = \E_{P^{j}_{\mathrm{pres}}}[V_j] - \E_{P^{j}_{\mathrm{abl}}}[V_j]$
      satisfies
      $\bigl|B_i^j\bigr| \;\leq\; \mathrm{osc}(V_j)\cdot \mathrm{TV}\bigl(P^{j}_{\mathrm{pres}}, P^{j}_{\mathrm{abl}}\bigr)$.
\item $B_i^j$ is a functional of the pair of laws and of $V_j$ alone. It does not depend on how many rollouts are drawn, so increasing the rollout count does not reduce it.
\end{enumerate}
\end{proposition}

\begin{proof}
\textit{Part (i).}
Let $E = \bigl(\operatorname{supp} P(o_j \mid \action_\emptyset)\bigr)^{c}$, the set of downstream observations that cannot arise under the null action. By definition $P^{j}_{\mathrm{abl}}(E) = 0$. Under $P^{j}_{\mathrm{pres}}$ the action is drawn from $\pib^i(\cdot \mid h_{u_i})$, and for every $\action \in S_j$ the law $P(o_j \mid \action)$ puts all its mass in $E$, so
\[
P^{j}_{\mathrm{pres}}(E) \;=\; \sum_{\action} \pib^i(\action \mid h_{u_i})\, P(o_j \mid \action)(E) \;\geq\; \pib^i(S_j \mid h_{u_i}).
\]
Since total variation distance is the supremum of $|P(A) - Q(A)|$ over events,
\[
\mathrm{TV}\bigl(P^{j}_{\mathrm{pres}}, P^{j}_{\mathrm{abl}}\bigr) \;\geq\; \bigl|P^{j}_{\mathrm{pres}}(E) - P^{j}_{\mathrm{abl}}(E)\bigr| \;=\; P^{j}_{\mathrm{pres}}(E) \;\geq\; \pib^i(S_j \mid h_{u_i}).
\]
For the verbatim case, \ref{ass:serial} gives a map $\varphi$ with $\varphi(o_j) = \action_{u_i}$ on the support of every $P(o_j \mid \action)$; observations produced under distinct actions therefore lie in disjoint sets, so $S_j \supseteq \cA \setminus \{\action_\emptyset\}$ and $\pib^i(S_j \mid h_{u_i}) \geq 1 - \pib^i(\action_\emptyset \mid h_{u_i})$, which is strictly positive by \ref{ass:nondegen}. This holds for every downstream decision point whose observation contains the message, whatever sampling occurs in between, because the argument uses only the marginal laws. At the opposite extreme, an agent whose message enters no downstream observation leaves $P(o_j \mid \action) = P(o_j \mid \action_\emptyset)$ for every action, so the two laws coincide, the shift is zero, and by part~(ii) so is the bias; a redundant agent is exactly this case.

\textit{Part (ii).}
Write $P = P^{j}_{\mathrm{pres}}$ and $Q = P^{j}_{\mathrm{abl}}$, and let
\[
m \;=\; \tfrac{1}{2}\bigl(\sup V_j + \inf V_j\bigr), \qquad W \;=\; V_j - m .
\]
Shifting by the midpoint centers the function: $\sup W = \tfrac{1}{2}\mathrm{osc}(V_j)$ and $\inf W = -\tfrac{1}{2}\mathrm{osc}(V_j)$, so $\norm{W}_\infty = \tfrac{1}{2}\mathrm{osc}(V_j)$, which is the smallest sup-norm any constant shift of $V_j$ can achieve. Because $P$ and $Q$ are both probability distributions, the constant contributes nothing to the difference of the two expectations,
\[
\sum_o m\,\bigl(P(o) - Q(o)\bigr) \;=\; m\,(1 - 1) \;=\; 0,
\]
and therefore
\[
B_i^j \;=\; \sum_o V_j(o)\,\bigl(P(o) - Q(o)\bigr) \;=\; \sum_o W(o)\,\bigl(P(o) - Q(o)\bigr).
\]
Bounding the right-hand side by H\"older's inequality and substituting the definition of total variation distance,
\[
\bigl|B_i^j\bigr|
\;\leq\; \norm{W}_\infty \sum_o \bigl|P(o) - Q(o)\bigr|
\;=\; \tfrac{1}{2}\mathrm{osc}(V_j)\cdot 2\,\mathrm{TV}(P, Q)
\;=\; \mathrm{osc}(V_j)\cdot\mathrm{TV}(P, Q).
\]
The midpoint shift is what makes the bound depend on the oscillation of $V_j$ rather than on its magnitude: without it the same argument would give the weaker $\norm{V_j}_\infty$ in place of $\tfrac{1}{2}\mathrm{osc}(V_j)$.

\textit{Part (iii).}
$B_i^j$ is defined as the difference of two expectations of $V_j$, one under each law. No sample size appears in that expression. Drawing more rollouts shrinks the Monte Carlo error of an \emph{estimate} of either expectation, but the quantity being estimated, and hence their difference, is unchanged.
\end{proof}

\begin{remark}
The two parts bound different things and can separate. The shift can be at its largest while the bias is zero: if $V_j$ is constant on the observations that ablation moves mass between, then $\mathrm{osc}(V_j)$ restricted to that mass is zero and $B_i^j = 0$ regardless of how large $\mathrm{TV}$ is. In the workflows of this paper the policy does not produce the null action, so the lower bound of part~(i) is close to one at every decision point and the bound in part~(ii) is driven by reward sensitivity alone. \Cref{app:ablation_bias} reports the realized bias by stratum of decision points.
\end{remark}

\subsection{Unbiasedness of the C3 Advantage}
\label{app:prop1}

We restate and prove \cref{prop:unbiased} from \cref{sec:loo}. Every expectation is taken at the recorded prefix $\tracepre$, of which $\huvar$ is a function; since the base run draws $\tracepre$ from its law given $\huvar$ and the alternatives are drawn from $\pib(\cdot \mid \huvar)$ independently of the rest of the prefix, an estimate that is unbiased at the prefix is unbiased for the $\huvar$-conditioned advantage of \eqref{eq:advantage} as well.

\begin{proposition}[Unbiasedness, restated]
When the trace is the state at $u$, with alternatives drawn independently from $\pib$, rollouts independent across alternatives under $\Db$, and non-adaptive rollout counts, the C3 advantage $\Ahat_j$ is, conditional on $\action_j$, an unbiased estimate of $A^*_{\Db}(\huvar, \action_j) = \QDb(\huvar, \action_j) - \VDb(\huvar)$, the expectation being over the rollouts and the other alternatives at the recorded prefix $\tracepre$.
\end{proposition}

\paragraph{Setup.}
Fix a decision point with history $\huvar$. A group of alternatives at that decision point contains $\fanout$ actions $\action_1, \ldots, \action_\fanout$ sampled independently from $\pib(\cdot \mid \huvar)$. Each action $\action_k$ receives $c_k \geq 1$ independent rollouts under $\Db$, producing rewards $\{R_{k,1}, \ldots, R_{k,c_k}\}$. The rollout counts $c_1, \ldots, c_\fanout$ do not depend on any sampled action or reward (non-adaptive allocation). Define:
\[
\hat{Q}^{\Db}(\huvar, \action_k) = \frac{1}{c_k} \sum_{m=1}^{c_k} R_{k,m}, \qquad C = \sum_{k=1}^\fanout c_k.
\]

\begin{proof}
We show $\E[\Ahat_j \mid \huvar, \action_j] = A^*_{\Db}(\huvar, \action_j)$, where the expectation is over rollout outcomes and the sampling of actions $\action_k$ for $k \neq j$ at the recorded prefix $\tracepre$; by \eqref{eq:sufficiency}, a rollout reset to $\tracepre$ follows the continuation law of the true state.

\textit{Step 1: Action value estimate.}
Since each $R_{k,m}$ is drawn independently from the terminal reward distribution conditioned on $(\huvar, \action_k)$ with continuation under $\pib$, and rollouts for distinct actions are independent:
\[
\E\bigl[\hat{Q}^{\Db}(\huvar, \action_k) \mid \action_k\bigr] = \QDb(\huvar, \action_k).
\]

\textit{Step 2: Baseline expectation.}
The baseline for action $j$ is $\bloo = \frac{1}{C - c_j} \sum_{k \neq j} c_k\, \hat{Q}^{\Db}(\huvar, \action_k)$. Taking expectation over rollout outcomes conditional on all actions:
\[
\E_{\text{rollouts}}\bigl[\bloo \mid \action_1, \ldots, \action_\fanout\bigr] = \frac{1}{C - c_j} \sum_{k \neq j} c_k\, \QDb(\huvar, \action_k).
\]
Next, take the expectation over the independent sampling of $\action_k \sim \pib(\cdot \mid \huvar)$ for $k \neq j$. Since $c_k$ are non-adaptive (fixed independently of the sampled actions), the factor $c_k / (C - c_j)$ is a constant that passes through the expectation:
\begin{align}
\E_{\action_k, k \neq j}\Bigl[\frac{1}{C - c_j} \sum_{k \neq j} c_k\, \QDb(\huvar, \action_k)\Bigr]
&= \frac{1}{C - c_j} \sum_{k \neq j} c_k\, \E_{\action_k \sim \pib}\bigl[\QDb(\huvar, \action_k)\bigr] \nonumber \\
&= \frac{1}{C - c_j} \sum_{k \neq j} c_k\, \VDb(\huvar) \nonumber \\
&= \VDb(\huvar). \label{eq:baseline_exp}
\end{align}
The second equality uses the definition $\VDb(\huvar) = \E_{\action \sim \pib}[\QDb(\huvar, \action)]$.

\textit{Step 3: Combining.}
\begin{align*}
\E[\Ahat_j \mid \huvar, \action_j]
&= \E\bigl[\hat{Q}^{\Db}(\huvar, \action_j) \mid \action_j\bigr] - \E[\bloo \mid \action_j] \\
&= \QDb(\huvar, \action_j) - \VDb(\huvar) \\
&= A^*_{\Db}(\huvar, \action_j).
\end{align*}
\end{proof}

\subsection{Variance Reduction via Within-Group Comparison}
\label{app:crn}

\Cref{sec:diagnostics} introduces within-group variance as a diagnostic metric. Here we derive how the within-group baseline eliminates between-group variation and express the residual variance as a function of the branching factor $\fanout$ and the rollout count $c_j$.

\paragraph{Decomposition.}
At a decision point with history $\huvar$, the empirical action value admits the decomposition:
\[
\hat{Q}^{\Db}(\huvar, \action_j) = \underbrace{\VDb(\huvar)}_{m(\huvar)} + \underbrace{A^*_{\Db}(\huvar, \action_j)}_{\delta_j} + \underbrace{\hat{Q}^{\Db}(\huvar, \action_j) - \QDb(\huvar, \action_j)}_{\varepsilon_j},
\]
where $m(\huvar)$ is the history-level mean (constant within a group), $\delta_j$ is the true action effect, and $\varepsilon_j$ is zero-mean rollout noise with $\text{Var}(\varepsilon_j \mid \action_j) = \sigma^2(\action_j) / c_j$.

\paragraph{Cancellation of between-group variation.}
The C3 advantage is:
\begin{align*}
\Ahat_j &= \hat{Q}^{\Db}(\huvar, \action_j) - \bloo \\
&= \bigl[m(\huvar) + \delta_j + \varepsilon_j\bigr] - \frac{1}{C - c_j}\sum_{k \neq j} c_k \bigl[m(\huvar) + \delta_k + \varepsilon_k\bigr] \\
&= \delta_j + \varepsilon_j - \frac{1}{C - c_j}\sum_{k \neq j} c_k (\delta_k + \varepsilon_k).
\end{align*}
The term $m(\huvar)$ cancels exactly because $\frac{1}{C - c_j}\sum_{k \neq j} c_k = 1$. This eliminates the dominant source of variance in multi-task training, where different task instances induce large variation in $\VDb(\huvar)$ across groups.

\paragraph{Residual variance.}
Conditional on all actions (hence all $\delta_k$ fixed), the variance of $\Ahat_j$ arises solely from rollout noise. Since rollouts for different actions are independent:
\begin{align}
\text{Var}(\Ahat_j \mid \text{actions}) &= \text{Var}(\varepsilon_j) + \frac{1}{(C-c_j)^2} \sum_{k \neq j} c_k^2\, \text{Var}(\varepsilon_k) \nonumber \\
&= \frac{\sigma^2}{c_j} + \frac{1}{(C-c_j)^2} \sum_{k \neq j} c_k^2 \cdot \frac{\sigma^2}{c_k} \nonumber \\
&= \frac{\sigma^2}{c_j} + \frac{\sigma^2}{(C-c_j)^2} \sum_{k \neq j} c_k \nonumber \\
&= \frac{\sigma^2}{c_j} + \frac{\sigma^2}{C - c_j}, \label{eq:var_ahat}
\end{align}
assuming homogeneous noise variance $\sigma^2(\action_k) \approx \sigma^2$. Under heterogeneous variance, the expression generalizes to $\sigma^2(\action_j)/c_j + (C-c_j)^{-2}\sum_{k \neq j} c_k\, \sigma^2(\action_k)$.

\paragraph{Equal-allocation case.}
When $c_k = c$ for all $k$ (uniform budget allocation), $C = \fanout c$ and:
\[
\text{Var}(\Ahat_j \mid \text{actions}) = \frac{\sigma^2}{c} + \frac{\sigma^2}{(\fanout - 1)c} = \frac{\sigma^2 \cdot \fanout}{c(\fanout - 1)}.
\]
For a fixed budget $B_u = \fanout \cdot c$ at the decision point (i.e., $c = B_u / \fanout$):
\[
\text{Var}(\Ahat_j \mid \text{actions}) = \frac{\sigma^2 \cdot \fanout^2}{B_u(\fanout - 1)}.
\]
The variance is minimized at $\fanout = 2$ (yielding $4\sigma^2 / B_u$) and increases monotonically for $\fanout > 2$, since $\frac{\partial}{\partial \fanout}\bigl[\fanout^2/(\fanout - 1)\bigr] = \fanout(\fanout - 2)/(\fanout - 1)^2 > 0$ for $\fanout > 2$.
\subsection{Estimating Influence from Groups of Alternatives}
\label{app:influence_derivation}

\Cref{sec:diagnostics} defines influence as the conditional mutual information $\influence$ between the injected action and the downstream response, where $J \in \{1, \ldots, \fanout\}$ indexes the injected action and $Y$ denotes the downstream response under $\Db$, the token sequence produced after the injection. Here we derive a plug-in estimator from the sampled groups of alternatives.

\paragraph{Definition.}
For a group of $\fanout$ alternatives at history $\huvar$, define $J$ as a random variable with distribution $P(J = k) = c_k / C$ (proportional to rollout allocation). The conditional mutual information is:
\[
I(J; Y \mid \huvar) = H(Y \mid \huvar) - H(Y \mid J, \huvar),
\]
where $H$ denotes Shannon entropy. When the downstream agent's response is deterministic given its full input (i.e., greedy decoding), $Y = f(J, \huvar)$ and $H(Y \mid J, \huvar) = 0$, so influence reduces to $H(Y \mid \huvar)$.

\paragraph{Discrete approximation.}
The terminal reward $R \in \{0, 1\}$ is a coarsened proxy for $Y$. The data processing inequality requires that $R$ depends on the upstream action $J$ only through the downstream response $Y$; formally, $J \to Y \to R$ forms a Markov chain conditioned on $\huvar$. This condition holds when the reward function scores only the final output without access to intermediate actions. Under this condition, $I(J; R \mid \huvar) \leq I(J; Y \mid \huvar)$, so the reward-level quantity is a lower bound on the response-level influence. For binary $R$:
\[
I(J; R \mid \huvar) = H(R \mid \huvar) - H(R \mid J, \huvar),
\]
where
\[
H(R \mid \huvar) = H\Bigl(\frac{1}{C}\sum_{k=1}^\fanout c_k \QDb(\huvar, \action_k)\Bigr)
\]
with $H(p) = -p \log p - (1{-}p)\log(1{-}p)$, and
\[
H(R \mid J, \huvar) = \frac{1}{C}\sum_{k=1}^\fanout c_k\, H\bigl(\QDb(\huvar, \action_k)\bigr).
\]

\paragraph{Plug-in estimator.}
Substituting the empirical success rates $\hat{Q}^{\Db}(\huvar, \action_k)$ of the group:
\begin{equation}
\label{eq:influence_est}
\hat{I}(J; R \mid \huvar) = H\Bigl(\frac{1}{C}\sum_{k} c_k \hat{Q}_k\Bigr) - \frac{1}{C}\sum_{k} c_k\, H(\hat{Q}_k),
\end{equation}
where $\hat{Q}_k = \hat{Q}^{\Db}(\huvar, \action_k)$. This estimator reuses the group data already collected for advantage estimation, requiring no additional sampling.

\paragraph{Finite-sample bias.}
The plug-in entropy estimator is negatively biased for small sample sizes~\citep{paninski2003estimation}. For binary outcomes with $c_k$ trials per action, the leading-order bias of each conditional entropy term is $-1/(2c_k)$, while the marginal entropy (estimated from the pooled $C$ observations) has bias $-1/(2C)$. Since $\hat{I} = \hat{H}(\text{marginal}) - \hat{H}(\text{conditional})$ and the conditional term is the more negatively biased of the two, the mutual information estimator is positively biased (overestimates dependence). The leading-order residual is:
\[
\E[\hat{I}] - I \;\approx\; -\frac{1}{2C} + \frac{1}{C}\sum_{k=1}^\fanout c_k \cdot \frac{1}{2c_k} \;=\; \frac{\fanout - 1}{2C}.
\]
This positive bias means $\hat{I}$ may slightly overestimate $I(J; R \mid \huvar)$, but since the true quantity is itself a lower bound on $I(J; Y \mid \huvar)$ via the data processing inequality, the combined estimate remains informative.

\subsection{Budget Allocation and Reliability}
\label{app:budget_resolution}

\Cref{sec:budget} states that the noise on an advantage in deeper workflows is set by the budget spent at that decision point and by the return variance rather than by the agent count. This subsection collects the arithmetic behind that claim from the variance result of \cref{app:crn} and states what it does and does not determine.

\paragraph{Variance at one decision point.}
Fix a decision point and let $B_u$ denote the budget available \emph{there}, spread over $\fanout$ alternatives with $c = B_u/\fanout$ rollouts each. Specializing the equal-allocation case of \cref{app:crn} to this budget,
\begin{equation}
\label{eq:var_budget}
\Var\bigl(\Ahat_j \mid \text{actions}\bigr)
\;=\; \frac{\sigma^2 \fanout}{c\,(\fanout - 1)}
\;=\; \frac{\sigma^2 \fanout^2}{B_u\,(\fanout - 1)} .
\end{equation}
As a function of $\fanout$ at fixed $B_u$ this is minimized at $\fanout = 2$, where it equals $4\sigma^2/B_u$, and rises for $\fanout > 2$, since $\partial_{\fanout}\bigl[\fanout^2/(\fanout-1)\bigr] = \fanout(\fanout-2)/(\fanout-1)^2 > 0$ there.

\paragraph{Depth divides the budget.}
Under a fixed total budget spread over the $\ndecisions$ decision points of an instance, each decision point receives, at one rollout per alternative, $\fanout = B/\ndecisions$ alternatives, so workflow depth enters only through this division. Two consequences follow directly. Below $\fanout = 2$ there is no contrast at all: the baseline $\bloo$ of \eqref{eq:bloo} averages over the peers of $\action_j$ within its group, and with a single alternative there are none, so the estimator is undefined rather than merely noisy. And a total budget of $B = 8$ crosses that line once a workflow has 5 or more decision points, which is the case the constant per-decision budget of \cref{sec:scaling} is introduced to avoid.

\paragraph{Role of Equation~\eqref{eq:var_budget}.}
Equation~\eqref{eq:var_budget} is not the objective. It gives the noise on a single advantage estimate, and it is monotone in $\fanout$ above 2, so on its own it would always recommend the smallest admissible branching factor. What it omits is that the estimate is used as a ranking: at $\fanout = 2$ the group offers one comparison and one peer to form the baseline from, and each additional alternative adds both comparisons to be settled and peers to average the baseline over. The two effects run against each other as $\fanout$ grows at fixed $B$. How the reliability of the ranking depends on $\fanout$ at a fixed budget is therefore an empirical question that \eqref{eq:var_budget} does not settle. The return variance $\sigma^2$ and the continuation count $c$ of \eqref{eq:var_budget} also enter the reliability the additive model implies. Write $S$ for the variance of the alternatives' true values at the decision point and $N = \sigma^2/c$ for the sampling variance of a $c$-continuation mean; a split-half estimate averages $c/2$ continuations per half, so each half carries noise variance $2N$, and the Pearson correlation between 2 halves whose noise is independent is $S/(S+2N)$ under an additive model. \Cref{tab:scaling} compares this prediction with the observed reliability of the 6 workflows.

\subsection{Influence at Three Levels: Reward, Answer, Response}
\label{app:influence_dist}

\Cref{app:influence_derivation} derives a plug-in estimator of the mutual information between the injected action and the downstream response, in which that response is coarsened to the terminal reward. The coarsening is what makes the estimator cheap, and it is also what limits what the estimator can see. This subsection states the uncoarsened quantity, and the answer-level quantity between the two, which \cref{app:ablation_bias} reports.

\paragraph{Definition.}
Let $J \in \{1, \ldots, \fanout\}$ index the alternative injected at decision point $u$, and let $Y$ be the downstream response itself, the full token sequence produced after the injection rather than any summary of it. \emph{Response-level influence} is the conditional mutual information
\begin{equation}
\label{eq:influence_dist}
\influence \;=\; H(Y \mid \huvar) \;-\; H(Y \mid J, \huvar),
\end{equation}
computed over the distribution of $Y$ under $\Db$. It is the reduction in uncertainty about what is written downstream that is attributable to knowing which upstream alternative it read. The name records the level at which the mutual information is taken: over the response distribution, rather than over the reward.

\paragraph{Relation to the reward-level estimator.}
The reward-level quantity of \cref{app:influence_derivation} replaces $Y$ by the terminal reward $R$. Conditioned on $\huvar$, the chain $J \to Y \to R$ is Markov whenever the reward function scores only the final output, which is the case throughout this paper, so the data processing inequality gives
\[
I(J; R \mid \huvar) \;\leq\; I(J; Y \mid \huvar).
\]
The reward-level value is therefore a lower bound on the response-level value, and the two are not on a common scale. A binary reward caps its own mutual information at one bit no matter how strongly the upstream choice determines the downstream text, so the reward-level estimator is blind to any change in downstream behavior that does not change whether the answer is right.

\paragraph{Estimation at the answer level.}
\Cref{app:ablation_bias} reports neither of the two quantities above but the one between them: $Y$ is taken to be the downstream final answer, a categorical variable read off the response, and $I(J; Y \mid \huvar)$ is estimated by the plug-in of \cref{app:influence_derivation} over the empirical answer distribution of each group, with the Miller-Madow correction~\citep{miller1955note} applied to each entropy term, in nats. Because the answer is a function of the response and the reward a function of the answer, the chain $J \to Y_{\mathrm{response}} \to Y_{\mathrm{answer}} \to R$ is Markov given $\huvar$, so the answer-level value sits between the reward-level and the response-level values.
\begin{table}[t]
\caption{Observed over predicted variance of the rollout noise on a C3 advantage, the observation being half the variance, over groups, of the difference between 2 advantage vectors computed from independent continuations of the same alternatives, and the prediction being $\sigma^2 \fanout^2 / (B_u(\fanout - 1))$ of \cref{app:budget_resolution} at the budget behind one vector, with $\sigma^2$ estimated from the same cell's continuation returns. Top: the 2-agent workflow on Qwen3-4B over a $3 \times 3$ grid of branching factors $\fanout$ and continuations per alternative $c$, the 2 vectors coming from 2 independent sets of $c$ continuations, so $B_u = \fanout c$; the grid is sampled independently of the rows below, so its $c = 4$ column and the 2-agent row are two independent estimates for the same design. Middle: the 6 workflows of \cref{sec:scaling} at every branching factor tested, 4 continuations per alternative, on Qwen3-4B, the 2 vectors coming from the odd and the even continuations, so $B_u = \fanout c/2$; $\ndecisions$ is the number of decision points. Bottom: at $\fanout = 4$, the variance across a decision point's 4 alternatives of their true values (signal, estimated as the variance of their 4-continuation mean returns minus the noise) and the sampling variance of a 4-continuation mean return (noise, the within-alternative variance of the 4 continuations divided by 4), each over the workflow's groups at its opening decision point; their ratio orders the 6 reliabilities of \cref{tab:scaling}. The rank correlation of the ratio with $\fanout$ in \cref{sec:scaling} combines the 30 configurations of the middle block with 13 on \secondmodel{}: the 2-agent and 3-agent chains at every branching factor, and the 4-turn, branching and 5-agent workflows at $\fanout = 4$.}
\label{tab:noise_law}
\centering
\small
\begin{tabular}{@{}lccc@{}}
\toprule
$\fanout$ & $c = 2$ & $c = 4$ & $c = 8$ \\
\midrule
2 & 0.92 & 0.99 & 1.04 \\
4 & 0.96 & 0.94 & 0.91 \\
8 & 1.01 & 1.05 & 1.03 \\
\bottomrule
\end{tabular}

\vspace{0.8em}
\begin{tabular}{@{}lccccc@{}}
\toprule
Workflow & $\fanout = 2$ & $\fanout = 3$ & $\fanout = 4$ & $\fanout = 6$ & $\fanout = 8$ \\
\midrule
2-agent chain, $\ndecisions = 2$ & 0.98 & 1.00 & 0.93 & 1.01 & 1.00 \\
3-agent chain, $\ndecisions = 3$ & 1.00 & 1.00 & 0.99 & 1.07 & 1.06 \\
4-turn workflow, $\ndecisions = 4$ & 0.97 & 1.01 & 1.05 & 1.07 & 1.00 \\
Branching workflow, $\ndecisions = 4$ & 0.90 & 0.93 & 1.00 & 0.99 & 1.03 \\
5-agent chain, $\ndecisions = 5$ & 1.00 & 1.02 & 0.95 & 0.98 & 0.99 \\
10-agent chain, $\ndecisions = 10$ & 1.07 & 0.99 & 1.03 & 1.02 & 0.98 \\
\bottomrule
\end{tabular}

\vspace{0.8em}
\begin{tabular}{@{}lcc@{}}
\toprule
Workflow ($\fanout = 4$) & Signal: variance of true values & Noise: variance of a mean \\
\midrule
2-agent chain, $\ndecisions = 2$ & 0.077 & 0.023 \\
3-agent chain, $\ndecisions = 3$ & 0.014 & 0.038 \\
4-turn workflow, $\ndecisions = 4$ & 0.031 & 0.016 \\
Branching workflow, $\ndecisions = 4$ & 0.025 & 0.033 \\
5-agent chain, $\ndecisions = 5$ & 0.011 & 0.041 \\
10-agent chain, $\ndecisions = 10$ & 0.003 & 0.049 \\
\bottomrule
\end{tabular}
\end{table}

\section{Supplementary Experiments and Reproducibility}
\label{app:supplementary}

The main text states when the trace of a cooperative LLM system is its state, which makes interventional credit assignment exact and lets credit quality be checked against executed counterfactuals. The rest of this appendix is organized as follows.
\Cref{app:notation} summarizes the main notation in a single reference table.
\Cref{app:hyperparams} specifies hyperparameters and training configuration.
\Cref{app:protocol} gives the complete prompt templates for the 2-agent and 3-agent workflows.
\Cref{app:extended_2a} gives every method of the long-generation setting, MAGRPO included, on the 4 main benchmarks and the 2 saturation controls, with the training tokens each method consumed, and the 3-agent workflow in the main setting.
\Cref{app:deep_protocols} describes the 4 deep workflows and the reliability estimates of \cref{sec:scaling}.
\Cref{app:ablation_bias,app:learned_critic} give the procedures behind \cref{sec:findings}: the per-decision ablation bias and the learned-critic comparison.
\Cref{app:significance} gives the significance testing procedure behind \cref{tab:main}.

\subsection{Notation}
\label{app:notation}

\Cref{tab:notation} summarizes the notation used throughout the paper. Symbols are grouped by topic.

\begin{table}[htbp]
\centering
\caption{Summary of notation.}
\label{tab:notation}
\small
\begin{tabular}{@{}l p{0.80\textwidth}@{}}
\toprule
\textbf{Symbol} & \textbf{Description} \\
\midrule
\multicolumn{2}{@{}l}{\textit{Problem Setting (\cref{sec:problem})}} \\
$N$ & Number of agents (\cref{fig:overview}); in \cref{sec:scaling}, \cref{app:budget_resolution} and \cref{tab:scaling}, $N$ is the sampling variance of a $c$-continuation mean ($c = 4$ throughout) and $S$ the variance of the alternatives' true values, distinct from the state $S_u$ of \cref{sec:sufficiency} and the message set $S_j$ of \cref{prop:distshift} \\
$\ndecisions$ & Number of decision points of a workflow \\
$u$ & Decision point index \\
$h_u$ & Context the acting role reads at decision point $u$ \\
$a$ & Action (complete textual message) \\
$\tau$ & Execution trace of one run \\
$R(\tau)$ & Terminal reward (scalar) \\
$\tracepre$ & Trace prefix: everything recorded before decision point $u$ \\
\midrule
\multicolumn{2}{@{}l}{\textit{Method (\cref{sec:method})}} \\
$\pib$ & Behavior policy: the snapshot of $\pitheta$ taken at each iteration and held fixed for the iteration \\
$\pitheta$ & Current learnable policy \\
$\Db$ & Continuation distribution under $\pib$ \\
$\fanout$ & Branching factor: alternatives sampled per decision point \\
$c_j$ & Number of Monte Carlo rollouts for action $j$ \\
$\QDb(h_u, a)$ & Action value under continuation distribution $\Db$ \\
$\VDb(h_u)$ & State value under continuation distribution $\Db$ \\
$\Astar(h_u, a)$ & Target advantage under $\Db$ \\
$\Ahat_j$ & C3 advantage estimate for action $j$ \\
$\bloo$ & Baseline for action $j$: the continuation-weighted mean return of the other alternatives \\
\midrule
\multicolumn{2}{@{}l}{\textit{Diagnostics (\cref{sec:diagnostics})}} \\
$\influence$ & Mutual information between the injected action and $Y$ \\
$J$ & Index of injected upstream action \\
$Y$ & Downstream token sequence after the injected action \\
$H(Y \mid h_u)$ & Conditional entropy of downstream response \\
\midrule
\multicolumn{2}{@{}l}{\textit{Training (\cref{sec:optimization})}} \\
$r(\theta)$ & Importance ratio of PPO, taken per token (\cref{app:hyperparams}) \\
$\epsilon$ & PPO clipping parameter \\
$\cL(\theta)$ & Policy gradient objective (clipped surrogate) \\
$B$ & Budget: rollouts per instance; $B_u$ is the budget at one decision point, both distinct from the ablation bias $B_i^j$ of \cref{prop:distshift} \\
\bottomrule
\end{tabular}
\end{table}

\FloatBarrier
\subsection{Hyperparameters and Training Configuration}
\label{app:hyperparams}

The complete hyperparameter specification for all experiments in \cref{sec:experiments} is listed below.

\paragraph{Shared Configuration.}
All methods (C3, MAPPO, MAGRPO) share the fixed settings listed in \cref{tab:hyperparams_shared}.

\begin{table}[htbp]
\centering
\caption{Training hyperparameters.}
\label{tab:hyperparams_shared}
\small
\begin{tabular}{@{}l p{0.60\textwidth}@{}}
\toprule
\textbf{Parameter} & \textbf{Value} \\
\midrule
Base model (trained workflows) & \texttt{Qwen3-4B-Instruct-2507} \\
Starting model & \texttt{Qwen3-4B-Instruct-2507} (the released instruct model, no further fine-tuning) \\
Learning rate (actor) & $1\!\times\!10^{-6}$ \\
Learning rate (critic, MAPPO only) & $5\!\times\!10^{-5}$ \\
Learning rate schedule & cosine, 3\% warmup ratio \\
Batch size (messages per optimizer step) & $256$ \\
Rollout batch size (instances) & $128$ \\
Micro-batch (train / rollout) & $64$ / $32$ (main); $4$ / $8$ (long generation) \\
PPO epochs per batch & $1$ \\
Passes over the training problems & $5$ (main); $2$ (long generation) \\
PPO clip ratio $\epsilon$ & $0.2$ \\
KL penalty coefficient (init) & $0.01$ \\
KL target & $0.1$ \\
KL estimator & \texttt{k3} \\
Prompt limit & $2560$ tokens \\
Generation limit & $512$ tokens (main); $2048$ (long generation) \\
Sampling & temperature $0.7$, top-$p$ $0.8$, top-$k$ $20$ \\
Optimizer & AdamW \\
Weight decay & $0.0$ \\
Gradient clipping norm & $1.0$ \\
Warmup steps & $3\%$ of total training steps \\
Precision & \texttt{bf16} \\
Rollout budget $B$ per instance & $8$ \\
Training problems (main) & MATH training split, 7{,}500 problems, disjoint from every evaluation set \\
Training problems (long generation) & 8{,}096 prompts, the MATH training split and the CMATH validation split, both disjoint from every evaluation set; 5{,}000 sampled per pass, 2 passes \\
Reward & $1$ if the final answer extracted from the last message matches the reference answer, else $0$ \\
Evaluated model & the final model of training \\
Evaluations per trained model & $1$ per seed (main, greedy decoding); $5$ (long generation) \\
Evaluation interval & every $5\%$ of training steps (main); every $25\%$ (long generation) \\
Hardware (main) & $8\!\times\!$NVIDIA A800-80GB PCIe, single node \\
Hardware (long generation) & $4\!\times\!$NVIDIA A100-80GB containers \\
\bottomrule
\end{tabular}
\end{table}

\Cref{tab:hyperparams_c3} lists the settings specific to C3.

\begin{table}[htbp]
\centering
\caption{C3-specific hyperparameters.}
\label{tab:hyperparams_c3}
\small
\begin{tabular}{@{}lp{0.62\linewidth}@{}}
\toprule
\textbf{Parameter} & \textbf{Value} \\
\midrule
Reasoner group per instance & $1$, with $\fanout = 2$ alternatives, each continued by its 4 Actor alternatives ($c_j = 4$) \\
Actor groups per instance & $2$, one per Reasoner alternative \\
Branching factor $\fanout$ per Actor group & $4$ \\
Rollouts per Actor alternative $c_j$ & $1$ \\
Reset & the trace prefix $\tracepre$ is kept at each decision point $u$ and the system is reset to it for counterfactual rollouts \\
Baseline & mean return of the other alternatives (the action's own return excluded) \\
Credit & returns only (no critic-based value subtraction) \\
\bottomrule
\end{tabular}
\end{table}

\paragraph{Policy objective as implemented.}
Each message is one training sample, a triple $(\huvar, \action_j, \Ahat_j)$ whose tokens are $\action_{j,1}, \ldots, \action_{j,|\action_j|}$. The importance ratio is taken per token, $r_{j,t}(\theta) = \pitheta(\action_{j,t} \mid \huvar, \action_{j,<t}) / \pib(\action_{j,t} \mid \huvar, \action_{j,<t})$, and on a micro-batch $\mathcal{M}$ of triples the objective maximized is
\begin{equation}
\cL(\theta) \;=\; \frac{\sum_{(\huvar, \action_j, \Ahat_j) \in \mathcal{M}} \sum_{t=1}^{|\action_j|} \min\bigl( r_{j,t}(\theta)\, \Ahat_j,\; \operatorname{clip}(r_{j,t}(\theta), 1-\epsilon, 1+\epsilon)\, \Ahat_j \bigr)}{\sum_{(\huvar, \action_j, \Ahat_j) \in \mathcal{M}} |\action_j|}, \qquad \epsilon = 0.2 .
\end{equation}
The advantage of a message is assigned to every token of $\action_j$ with no division by the message length; prompt and padding positions are masked out, and the mean runs over all unmasked tokens of the micro-batch. The log-probabilities under $\pib$ and all advantages are computed once per iteration, before any gradient step, and each batch is used for one PPO epoch; $\pitheta$, and with it $r_{j,t}(\theta)$, is re-evaluated on every micro-batch. The KL penalty of \cref{tab:hyperparams_shared} is estimated against the starting model with $k_3 = (r - 1) - \log r$, where $r$ is the probability ratio of the starting model to the policy, with a coefficient that starts at $0.01$ and is adapted once per iteration toward a target of $0.1$. Both roles share the parameters $\theta$ and this single objective, with no role-dependent weight.

\paragraph{Baseline Implementation Details.}
MAPPO uses a centralized critic (\texttt{mmBERT-base}, separate model; \citealp{marone2025mmbert}) that receives the concatenation of all agents' observations, with a maximum state length of 2560 and a critic warmup of 64 steps; its per-decision advantages are standardized over the global batch and divided by the message's token count. MAGRPO computes group-relative advantages with the group mean as baseline, the joint action as the unit of advantage, and 8 samples per group. Both baselines draw 8 independent full-episode rollouts per instance, the budget $B = 8$ that C3 spends; all three methods share the 256 messages per optimizer step.

\paragraph{Hardware.}
The main setting uses a single node with $8\!\times\!$NVIDIA A800-80GB PCIe GPUs; the long-generation setting uses cloud containers with 4 A100-80GB GPUs, and the critic comparison uses containers with 2 GPUs.

\FloatBarrier
\subsection{Workflow Specification}
\label{app:protocol}

The complete prompt templates of the 2-agent workflow of \cref{sec:diagnostics} and the 3-agent workflow of \cref{sec:scaling,sec:findings} follow.

\paragraph{2-agent workflow (Reasoner $\to$ Actor).}
The 2-agent workflow assigns 2 roles to 2 separate model instances. The Reasoner receives the task input and produces a step-by-step solution plan. The Actor receives both the task input and the Reasoner's plan, then produces the final answer.

\begin{PromptBox}{Reasoner System Prompt (2-agent workflow)}
Two LLM agents (Reasoner $\to$ Actor) collaborate step-by-step to solve math problems. You are the \textbf{Reasoner}: Analyze the original problem, historical actions, and reflection data (if provided) to determine the critical next step. Guide the Actor by providing concise reasoning for the optimal operation.

\smallskip\noindent The Reasoner gives no final answer.
\end{PromptBox}

\begin{PromptBox}{Actor System Prompt (2-agent workflow)}
Two LLM agents (Reasoner $\to$ Actor) collaborate step-by-step to solve math problems. You are the \textbf{Actor}: Execute operations using the original problem, action history, and the Reasoner's guidance. Provide the final answer within \texttt{\textbackslash boxed\{\}}.

\smallskip\noindent The Actor gives the final answer and reads the Reasoner's message.
\end{PromptBox}

\paragraph{3-agent workflow (Reasoner $\to$ Actor $\to$ Verifier).}
The 3-agent workflow extends the 2-agent workflow with a third agent. The Reasoner prompt is unchanged. The Actor prompt is modified to defer the final answer to the Verifier:

\begin{PromptBox}{Actor System Prompt (3-agent workflow)}
Three LLM agents (Reasoner $\to$ Actor $\to$ Verifier) collaborate step-by-step to solve math problems. You are the \textbf{Actor}: Produce a candidate step-by-step solution using the original problem, action history, and the Reasoner's guidance. Focus on correctness and leave the final checked answer to the Verifier.

\smallskip\noindent The Actor gives no final answer and reads the Reasoner's message.
\end{PromptBox}

\begin{PromptBox}{Verifier System Prompt (3-agent workflow)}
Three LLM agents (Reasoner $\to$ Actor $\to$ Verifier) collaborate step-by-step to solve math problems. You are the \textbf{Verifier}: Review the Actor's candidate solution, cross-check it by alternative methods or by plugging the result back in, and then provide the final checked answer. Output the final answer within \texttt{\textbackslash boxed\{\}}.

\smallskip\noindent The Verifier gives the final answer and reads the Actor's message.
\end{PromptBox}

\paragraph{Context Serialization.}
At each decision point $u$, the input to the downstream agent is constructed by Python format-string substitution into the role's system prompt. The following placeholders are available:

\begin{PromptBox}{Context Serialization}
\texttt{\{question\}}\quad Original task input (math problem text).\\
\texttt{\{context\}}\quad \texttt{"\textbackslash n\textbackslash n"}-joined outputs of all upstream roles in topological order.\\
\texttt{\{<role>\}}\quad Direct access to a specific upstream role's output (e.g., \texttt{\{reasoner\}}, \texttt{\{actor\}}).

\medskip\noindent\textbf{Execution order:}\\
\textit{2-agent:} Reasoner $\leftarrow$ \texttt{\{question\}}; \quad Actor $\leftarrow$ \texttt{\{question\}} + \texttt{\{reasoner\}}.\\
\textit{3-agent:} Reasoner $\leftarrow$ \texttt{\{question\}}; \quad Actor $\leftarrow$ \texttt{\{question\}} + \texttt{\{reasoner\}}; \quad Verifier $\leftarrow$ \texttt{\{question\}} + \texttt{\{reasoner\}} + \texttt{\{actor\}}.

\medskip\noindent Missing keys are replaced with the empty string. Substitution occurs once per decision point; outputs are appended to \texttt{\{context\}} for downstream roles.
\end{PromptBox}

\FloatBarrier
\subsection{Extended Training Results}
\label{app:extended_2a}

\Cref{tab:saturation_controls} gives the long-generation setting in full: every method, including MAGRPO's final model and its model before the divergence, on the 4 main benchmarks, MATH500, Minerva-Math, AMC23 and AIME, and on the 2 saturation controls, CMATH and GSM8K, where the starting model is already within a few points of the ceiling. On the 4-benchmark average the starting model, MAPPO and C3 reach 63.0, 64.2 and 65.3. Training consumed 191.9M tokens for MAPPO, 213.6M for MAGRPO and 120.0M for C3, with 2 passes over the training problems and 8 rollouts to the terminal reward per instance for every method.

\begin{table}[htbp]
\centering
\caption{The long-generation setting on the 2-agent workflow with Qwen3-4B: accuracy (\%) under sampled decoding, avg@$k$ with $k = 16$ on AMC23 and AIME (2024 and 2025 merged) and $k = 4$ on the other 4 benchmarks, mean $\pm$ std over 5 evaluations of one model, on the 4 main benchmarks and on the 2 saturation controls, CMATH and GSM8K, where the starting model already sits near the ceiling; the last column is the training tokens each method consumed, in millions. $^\dagger$MAGRPO at step 250, before its divergence.}
\label{tab:saturation_controls}
\footnotesize
\setlength{\tabcolsep}{2.5pt}
\begin{tabular}{@{}lccccccc@{}}
\toprule
\textbf{Method} & \textbf{MATH500} & \textbf{Minerva} & \textbf{AMC23} & \textbf{AIME} & \textbf{CMATH} & \textbf{GSM8K} & \textbf{Tokens (M)} \\
\midrule
Starting model & 88.5\,$\pm$\,0.2 & 59.2\,$\pm$\,0.8 & 75.8\,$\pm$\,1.3 & 28.6\,$\pm$\,0.4 & 96.2\,$\pm$\,0.2 & 93.8\,$\pm$\,0.1 & 0 \\
MAPPO & 88.9\,$\pm$\,0.3 & 59.6\,$\pm$\,0.5 & 78.3\,$\pm$\,0.7 & 29.9\,$\pm$\,0.8 & 96.4\,$\pm$\,0.1 & 93.6\,$\pm$\,0.1 & 191.9 \\
MAGRPO, final & 1.9\,$\pm$\,0.3 & 1.0\,$\pm$\,0.6 & 1.0\,$\pm$\,0.1 & 0.1\,$\pm$\,0.1 & 34.4\,$\pm$\,1.2 & 2.4\,$\pm$\,0.3 & 213.6 \\
\quad step 250$^\dagger$ & 89.0\,$\pm$\,0.2 & 60.0\,$\pm$\,0.2 & 76.9\,$\pm$\,0.4 & 30.1\,$\pm$\,0.8 & 96.4\,$\pm$\,0.0 & 93.7\,$\pm$\,0.1 & \\
C3 & 89.5\,$\pm$\,0.1 & 59.8\,$\pm$\,0.6 & 81.1\,$\pm$\,1.2 & 30.6\,$\pm$\,0.6 & 96.3\,$\pm$\,0.2 & 93.8\,$\pm$\,0.2 & 120.0 \\
\bottomrule
\end{tabular}
\end{table}

We also train the 3-agent workflow, the Reasoner, Actor and Verifier chain of \cref{sec:diagnostics}, in the main setting with the same hyperparameters as the 2-agent workflow; \cref{tab:main_3a} reports it with the training tokens of each method. C3 leads on MATH500 and on the 3-benchmark average, and its training consumed 418.4M tokens against 616.1M for MAPPO and 703.1M for MAGRPO, 32\% and 40\% fewer.

\begin{table}[htbp]
\centering
\caption{Main setting, the 3-agent workflow of Qwen3-4B (the Reasoner, Actor and Verifier chain of \cref{sec:diagnostics}): greedy accuracy (\%) under a 512-token generation limit, mean $\pm$ std over 5 training seeds, and the training tokens each method consumed; best value per column in bold. $\dagger$Mean of MATH500, CMATH and GSM8K; AIME 2025 (30 problems) is not included in the average.}
\label{tab:main_3a}
\footnotesize
\setlength{\tabcolsep}{3.5pt}
\begin{tabular}{@{}lcccccc@{}}
\toprule
Method & MATH500 & AIME 2025 & CMATH & GSM8K & Avg.$^{\dagger}$ & Tokens (M) \\
\midrule
MAPPO & 75.8\,$\pm$\,0.7 & 6.7\,$\pm$\,1.4 & 95.8\,$\pm$\,0.3 & 93.4\,$\pm$\,0.3 & 88.3\,$\pm$\,0.3 & 616.1 \\
MAGRPO & 62.1\,$\pm$\,1.3 & 7.3\,$\pm$\,2.5 & 79.6\,$\pm$\,0.8 & 92.9\,$\pm$\,0.3 & 78.2\,$\pm$\,0.5 & 703.1 \\
C3 & \textbf{78.6\,$\pm$\,0.5} & \textbf{7.8\,$\pm$\,2.5} & \textbf{95.9\,$\pm$\,0.2} & \textbf{93.5\,$\pm$\,0.3} & \textbf{89.3\,$\pm$\,0.4} & \textbf{418.4} \\
\bottomrule
\end{tabular}
\end{table}

\FloatBarrier
\subsection{Deep Workflow Specifications and Budget Rules}
\label{app:deep_protocols}

The depth study of \cref{sec:scaling} uses 4 workflows beyond the 2-agent and 3-agent topologies of \cref{app:protocol}. All of them reuse the context serialization of \cref{app:protocol} unchanged: each role's prompt is assembled by substitution of the task text and the outputs of the roles it depends on, in topological order, so the trace prefix records the entire input to every decision point.

\paragraph{5-agent chain ($\ndecisions = 5$).}
A linear chain of 5 roles, Planner $\to$ Solver $\to$ Critic $\to$ Reviser $\to$ Verifier, each depending on its predecessor. The Planner produces a solution plan without solving; the Solver executes it; the Critic lists concrete defects without rewriting the solution; the Reviser corrects the defects; the Verifier cross-checks and emits the final answer. Only the Verifier emits an answer in \verb|\boxed{}|.

\paragraph{10-agent chain ($\ndecisions = 10$).}
A linear chain of 10 roles: Reader $\to$ Planner $\to$ Decomposer $\to$ SolverA $\to$ SolverB $\to$ Integrator $\to$ Critic $\to$ Reviser $\to$ Checker $\to$ Verifier. The Reader restates the problem, its givens, and what is asked; the Planner produces an ordered plan; the Decomposer splits that plan into 2 work packages, an opening package~A and a remaining package~B; SolverA executes package~A and SolverB continues from its partial working to execute package~B; the Integrator merges the 2 partial workings into one coherent candidate solution; the Critic lists concrete defects; the Reviser corrects them; the Checker cross-checks by an independent route and reports whether the check passes; the Verifier resolves any remaining discrepancy and emits the final answer.

\paragraph{4-turn workflow ($\ndecisions = 4$).}
The Reasoner and the Actor alternate twice (Reasoner, Actor, Reasoner, Actor), so that the agent count is held at the 2-agent value while decision depth rises to 4. This workflow separates temporal depth from cross-agent depth: it is the control for reading the 3-agent result, since the 2-agent and 3-agent workflows differ in agent count and depth at once.

\paragraph{Branching workflow ($\ndecisions = 4$).}
A Planner writes the plan; 2 Solvers work from the task text and the plan in parallel, neither seeing the other; an Integrator reads both workings and emits the final answer. It is the one workflow in this paper whose wiring is not a chain: the 2 Solver decision points share a prefix and are re-run independently of each other, and the Integrator's context is assembled from both, so the trace prefix of every decision point is still the whole of what that role reads.

\paragraph{The two quantities.}
Two quantities are reported per workflow in \cref{tab:scaling}, both computed from the groups of alternatives already collected for advantage estimation and both on a policy held fixed, so that they describe the estimator rather than a training outcome. Every workflow is evaluated at its opening decision point (the Reasoner, Planner or Reader), on the released instruct model held fixed, with a 512-token limit per message.

\begin{AppItemize}
\item \emph{Split-half reliability.} Within a group, the rollouts of each alternative are partitioned into 2 halves, the even-numbered and the odd-numbered continuations, the same halves as in \cref{tab:noise_law}. Each half gives its own mean return per alternative and therefore its own advantage vector over the $\fanout$ alternatives of the group (\cref{eq:bloo}). The split-half reliability of the group is the Spearman rank correlation between the 2 vectors, and the reported value is the mean over groups. The quantity needs no separate reference rollouts, which is what makes it comparable across workflows of different depth. Groups with fewer than 3 alternatives, and groups in which either half assigns the same advantage to every alternative (for example because every rollout of that half returned the same reward) have no ranking to compare and are excluded from the mean; the statistic therefore describes the decision points at which the outcome varies across alternatives, and the excluded fraction grows with the share of instances the policy either always solves or never solves. Under the fixed budget it is 35\% and 50\% at 2 and 3 decision points (at 4 decision points, the fixed budget leaves 2 alternatives per group, below the 3 alternatives the statistic needs, so no value is defined there). Because subtracting the mean of the other alternatives is an increasing affine map of the per-alternative mean return within a group, the rank correlation is that of the per-alternative mean returns themselves.

\item \emph{Within-group variance.} The sample variance of the advantages across the alternatives of a group, averaged over groups, as defined in \cref{sec:wgv}.
\end{AppItemize}

\paragraph{Two budget rules.}
Under a \emph{fixed total budget}, one budget $B$ is spent per instance regardless of depth, so each decision point receives a branching factor of $\fanout = B/\ndecisions$ (\cref{sec:budget}); C3's training in \cref{sec:experiments} spends its $B = 8$ differently, as one $2 \times 4$ tree whose 8 Actor rollouts are also the Reasoner's continuations (\cref{tab:hyperparams_c3}). At $\ndecisions = 4$ with $B = 8$ this leaves $\fanout = 2$, the smallest value at which another alternative exists to compare against, and at 5 or more decision points the contrast is undefined. Under a \emph{constant per-decision budget}, the budget scales with depth as $B = 4\ndecisions$, holding $\fanout = 4$ at every decision point: $B = 16$ for the 4-turn and branching workflows, $B = 20$ for the 5-agent chain, and $B = 40$ for the 10-agent chain. Both rules count alternatives, at one rollout each; the estimates then give every alternative 4 continuations, so that the split-half statistic is defined. Under the fixed budget the 4-turn and branching workflows would have 2 alternatives per group, and the deeper chains fewer, below the 3 alternatives the statistic needs, so \cref{tab:scaling} reports every workflow at the constant per-decision budget and \cref{fig:reliability} sweeps the branching factor.

\begin{figure}[t]
  \centering
  \includegraphics[width=0.6\textwidth]{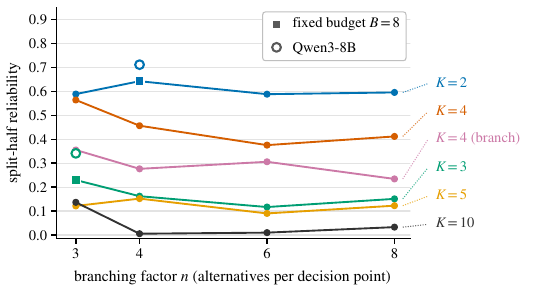}
  \caption{Split-half reliability on a policy held fixed against the branching factor $\fanout$, one line per workflow labeled with its number of decision points $\ndecisions$; every alternative receives 4 continuations so that the statistic is defined. Filled squares: an even split of 8 alternatives over the decision points, 4 per point on the 2-agent chain and 3 on the 3-agent chain. Open circles: the same fixed-budget points on \secondmodel{}.}
  \label{fig:reliability}
\end{figure}

\paragraph{Model scale.}
On \secondmodel{}, split-half reliability is 0.71 for the 2-agent workflow and 0.34 for the 3-agent workflow under the fixed total budget $B = 8$, and 0.71 and 0.26 under the constant per-decision budget.

\FloatBarrier
\subsection{Ablation Bias at Each Decision Point}
\label{app:ablation_bias}

\Cref{prop:distshift} bounds the ablation bias by $\mathrm{osc}(V_j)\cdot\mathrm{TV}$ but does not say how much of that bound is realized. This subsection gives the procedure behind \cref{tab:ablation_bias}.

\paragraph{Sample.}
The sample is 150 decision points under the 3-agent workflow on MATH500, the Reasoner position of the initial 150 questions in a fixed order, one per question, so that the question is the independent unit for the paired test reported below. Each point carries 4 alternatives with 4 rollouts each, plus rollouts under the null action. The 16 rollouts are those of advantage estimation; only the null-action rollouts are added.

\paragraph{Two operationalizations of the null action.}
Ablation has to be turned into something a language model can be run on, and there are two natural choices. The null action $\action_\emptyset$ is either the deletion of the role's entire segment from the downstream context (under the assembly map $\psi$ of \cref{sec:posg}, an empty message and a deleted segment produce the same downstream prompt, so this is the empty-message form), or a fixed placeholder message that keeps the role present in the context but contributes nothing. The two are not interchangeable in principle: one erases the fact that the role acted at all, the other leaves it present but silent. On the same decision points the two forms do not differ (paired test, \mbox{$p$ = 0.92}); the results below use the empty-message form, the counterfactual that CCPO~\citep{wan2025ccpo} subtracts.

\paragraph{Per-decision-point bias.}
At each decision point the computed quantity is the mean return over the alternatives actually sampled from $\pib$ at that point, minus the return obtained under the null action, both estimated from the rollouts of that decision point. This is the orientation of \cref{prop:distshift}, which defines $B_i^j$ as the continuation value with the agent present minus the value with it ablated, so a positive value means that ablation costs return.

\paragraph{Two diagnostics related to the bound.}
For reward sensitivity we use the empirical oscillation of the estimated action values within a group of alternatives,
\[
\hat{\Delta}_{Q,u} \;=\; \max_k \hat{Q}^{\Db}(\huvar, \action_k) \;-\; \min_k \hat{Q}^{\Db}(\huvar, \action_k),
\]
which is a within-group spread over the sampled alternatives at $u$. It is used only to split decision points into a stratum where the estimated action values differ and a stratum where they do not; it is not an estimator of the population quantity $\mathrm{osc}(V_j)$. For downstream dependence we use the answer-level plug-in influence of \cref{app:influence_dist}: the conditional mutual information of \cref{sec:influence} with $Y$ taken to be the downstream final answer, estimated on the same groups. Its mean over these groups is 0.36 nats.

\paragraph{Stratification.}
The 150 points are split at the median influence into a high half and a low half. Within the high half, the points are further split by whether $\hat{\Delta}_{Q,u}$ is nonzero ({}52 points) or zero ({}23 points). \Cref{tab:ablation_bias} reports the two strata: 96\% of the points whose alternatives received equal estimated values have exactly zero bias, against 3 of the 52 points with differing values (the points with equal estimated values are almost all instances that every rollout solves or fails, where the sampled alternatives and the null action return the same reward). The low-influence half carries a mean absolute bias of 0.172.

\paragraph{Aggregate and concentration.}
The 52 points that are high on both factors, 35\% of the sample, carry 44\% of the total absolute bias over the 150 points, a mean absolute bias of 0.20 against 0.135 for the remaining 98 points.

\begin{table}[t]
\caption{Realized agent-ablation bias, stratified by reward sensitivity within the high-influence half of the sample (answer-level plug-in influence, \cref{app:influence_dist}). The sample is 150 decision points of the 3-agent workflow on MATH500, one per evaluation question; the high-influence half is the half above the median influence, and the two strata split those points by whether the sampled alternatives received different estimated values ($\hat{\Delta}_{Q,u} > 0$). The two operationalizations of the null action $\action_\emptyset$ (an empty message and a fixed placeholder message) differ by no detectable amount (paired $p$ = 0.92).}
\label{tab:ablation_bias}
\centering
\footnotesize
\begin{tabular}{@{}lccccc@{}}
\toprule
Stratum & Points & Mean bias & Median & IQR & Exactly zero \\
\midrule
Alternatives differ in value ($\hat{\Delta}_{Q,u} > 0$) & 52 & \ensuremath{+0.02} & \ensuremath{+0.06} & [\ensuremath{-0.08}, \ensuremath{+0.19}] & 3\,/\,52 \\
Alternatives equal in value ($\hat{\Delta}_{Q,u} = 0$) & 23 & \ensuremath{-0.01} & \ensuremath{+0.00} & [\ensuremath{+0.00}, \ensuremath{+0.00}] & 22\,/\,23 \\
\bottomrule
\end{tabular}
\end{table}

\FloatBarrier
\subsection{Learned-Critic Comparisons}
\label{app:learned_critic}

C3 is compared with 2 learned critics, at 2 different granularities. The distinction matters, because the two answer different questions and only one of them is a competitor for per-decision credit.

\paragraph{State-value critic.}
MAPPO carries a centralized state-value critic $V(\huvar)$ over the concatenated observations of all agents (\cref{app:hyperparams}). As a credit estimator it has a structural limit that no training budget lifts: a state-value critic assigns one number to a decision point, so subtracting it shifts every alternative at that decision point by the same constant and leaves their relative order untouched. It therefore adds no within-group information beyond the returns of the rollouts it is given; what it changes is the between-group centering, and MAPPO's own training builds no shared-history groups at all.

\paragraph{Decision-value critic.}
The estimator that does compete at the right granularity is a decision-value critic $\hat{Q}(\huvar, \action)$ that scores each alternative individually. It is a regression head on a small pretrained encoder (mmBERT-base) that reads the decision point's context and the alternative as one sequence of up to 1024 tokens, and it is trained on the rollout data C3 already collects: the groups of the 2-agent workflow of \cref{sec:scaling} on the policy held fixed, 500 decision points with 4 alternatives and 4 continuations each, of which 400 decision points train the critic and 100 are held out, split by decision point so that no alternative of a held-out point enters training. The target is the mean return of an alternative's continuations. The family is the learning rate in $\{10^{-5}, 3 \times 10^{-5}, 10^{-4}\}$ crossed with 2 heads, a linear output and one hidden layer of 768 units; each of the 6 is trained for 3 epochs (AdamW, batch size 16, weight decay 0.01), the one with the lowest validation loss on a fifth of the training questions is retrained on all training points.

\paragraph{Evaluation.}
Both estimators are scored on the held-out decision points against a Monte Carlo reference of 16 fresh continuations per alternative, drawn with a separate random seed, which is unbiased under the condition of \cref{sec:sufficiency}. The metric is the \emph{rank correlation on held-out decision points}: the Spearman rank correlation between an estimator's ordering of the 4 alternatives at a decision point and the reference ordering, averaged over the points at which the reference and both estimators order the alternatives, 79 of the 100 held-out points (a point drops out when the reference or an estimator assigns all 4 alternatives the same value, the rule of the split-half statistic in \cref{sec:diagnostics}). It is not the split-half reliability of \cref{sec:diagnostics}, which uses the rollouts themselves as their own reference. C3's estimate at a held-out point is the advantage from its own 4 continuations per alternative, so the two estimators are compared on the same points against a reference independent of both.

\paragraph{Result.}
The decision-value critic reaches a rank correlation of 0.29 on the held-out decision points; C3 on the same points reaches 0.69 from 4 continuations per alternative. The paired difference is \ensuremath{+0.40} at $p$ < 0.001. A second independent draw of 16 fresh continuations per alternative on the same held-out alternatives agrees with the reference at 0.73 on 88 points, which is the agreement of 2 independent draws of the reference. On the  75 held-out points where the reference, C3 with 4 continuations and the second 16-continuation reference at 4, 8 and 16 continuations all rank the alternatives, C3 reaches  0.70 and the 2 references agree at  0.76. Retrained with the same settings on 2 and 4 times the rollouts, the critic stays well below C3's agreement with the reference: over the three data sizes the upper end of its 95\% bootstrap interval over held-out points reaches 0.45 at most, while the interval of C3 from 4 continuations starts at 0.59. Trained on the same rollouts, the critic falls well short of ranking alternatives within a decision point, which is the operation per-decision credit assignment consists of.

\FloatBarrier
\subsection{Significance Testing Procedure}
\label{app:significance}

An average over benchmarks can be moved by one benchmark, so the MATH500 gain in the main setting is tested on its own with a paired test. This subsection gives the procedure.

\paragraph{Pairing and sample.}
The comparison is C3 against MAGRPO, the baseline with the highest MATH500 mean in \cref{tab:main}. The test uses the per-seed results of \cref{tab:main}, paired by seed index over the 5 training seeds of each method.

\paragraph{Test.}
The paired differences (C3 minus the strongest baseline, in percentage points) are tested against zero with a paired two-sided $t$ test, and the reported interval is the corresponding 95\% confidence interval on the mean difference.
\paragraph{Multiplicity.}
The 4 benchmarks tested at once, the MATH500, AIME 2025, CMATH and GSM8K columns of \cref{tab:main}, inflate the chance that one of them clears an uncorrected threshold, so the 4 raw $p$ values are corrected by the Holm step-down procedure within the family. Holm sorts the raw values in ascending order and scales the $k$-th smallest of $m$ by $m - k + 1$, taking a running maximum so that the corrected values stay ordered. It controls the family-wise error rate under arbitrary dependence between the tests, which matters here because the 4 benchmarks are scored on the same trained models and their differences are not independent. The MATH500 difference is \ensuremath{+8.28}\,pp, 95\% interval \ensuremath{+7.48} to \ensuremath{+9.08}, at a Holm-corrected $p$ < 0.001; \cref{tab:holm} lists all 4 differences with their intervals and both $p$ values.

\begin{table}[htbp]
\centering
\caption{Main setting: paired differences, C3 minus MAGRPO, on the 4 benchmarks tested as one family, in points, with 95\% intervals and raw and Holm-corrected $p$ values over $n = 5$ training seeds. CMATH and GSM8K are the saturation controls of \cref{app:extended_2a}, on which the starting model already sits near the ceiling; AIME 2025 has 30 problems.}
\label{tab:holm}
\footnotesize
\begin{tabular}{@{}lccll@{}}
\toprule
Benchmark & Difference & 95\% interval & Raw & Holm \\
\midrule
MATH500 & \ensuremath{+8.28} & \ensuremath{+7.48} to \ensuremath{+9.08} & $p$ < 0.001 & $p$ < 0.001 \\
AIME 2025 & \ensuremath{+2.67} & \ensuremath{-0.80} to \ensuremath{+6.13} & $p$ = 0.10 & $p$ = 0.30 \\
CMATH & \ensuremath{+0.26} & \ensuremath{-0.19} to \ensuremath{+0.70} & $p$ = 0.19 & $p$ = 0.38 \\
GSM8K & \ensuremath{+0.09} & \ensuremath{-0.33} to \ensuremath{+0.51} & $p$ = 0.58 & $p$ = 0.58 \\
\bottomrule
\end{tabular}
\end{table}

\FloatBarrier
\subsection{What Each Credit Signal Restores, Re-runs and Compares}
\label{app:mechanisms}

\Cref{tab:mechanisms} places the credit signals of \cref{sec:related} side by side at the level of mechanism: the unit that receives a number, what is restored before an alternative is valued, how its value is obtained, and what that value is compared against. Of these, C3 alone values every sampled alternative at every role of a team by executing its continuation to the terminal reward and compares it with the other alternatives sampled at the same decision point; the single-agent methods re-run from a prefix for one policy, and the multi-agent methods either predict the value, share one score, remove an agent or a message, or score a candidate without continuing it.

\begin{table}[htbp]
\centering
\caption{Mechanisms of the credit signals compared in \cref{sec:related}: what each restores before it values an alternative, how it obtains the value, and what it compares the value against.}
\label{tab:mechanisms}
\footnotesize
\setlength{\tabcolsep}{3pt}
\begin{tabular}{@{}p{0.17\textwidth} p{0.15\textwidth} p{0.19\textwidth} p{0.21\textwidth} p{0.19\textwidth}@{}}
\toprule
\textbf{Method} & \textbf{Unit credited} & \textbf{Restored before valuing} & \textbf{Value obtained by} & \textbf{Compared against} \\
\midrule
MAPPO~\citep{yu2022surprising} & each decision point & nothing: the value is predicted & a centralized critic on the concatenated observations & the critic's value of the state at that decision point \\
MAGRPO~\citep{liu2025magrpo} & every decision of a trajectory, the same number & nothing & the trajectory's terminal score & the mean score of a group of trajectories \\
CCPO~\citep{wan2025ccpo} & one agent of a 2-agent chain & the downstream agent's prompt with the ablated message removed & re-running the downstream agent to the reward & the reward with the agent present (\cref{prop:distshift}) \\
MALT~\citep{motwani2024malt} & each role's output & the prefix, branched into a tree of sampled continuations & the outcomes below the output, turned into labels & a label threshold; training is off-policy on the labels \\
\midrule
VinePPO~\citep{kazemnejad2024vineppo} & a reasoning step of one agent & the step's prefix & Monte Carlo completions from the prefix & the Monte Carlo value of the preceding step \\
Tree-GRPO~\citep{ji2025treegrpo} & a step of one agent & the step's prefix, branched into a tree & the terminal outcomes of the branches & the group mean over the tree's branches \\
ROME~\citep{wang2025rome} & a tool-call chunk of one trajectory & nothing & the trajectory's terminal return, discounted to the chunk & no per-decision baseline \\
\midrule
AT-GRPO~\citep{pham2026stronger} & one agent's candidate at one turn & the same context, with $K$ candidates sampled & an immediate verifiable reward of the candidate; only the best candidate is continued & the mean over the $K$ candidates \\
Re-run without one message~\citep{yang2026wrongbutuseful} & one message & the trajectory without that message & re-running the trajectory & the reward with the message present \\
\midrule
C3 (this paper) & a decision point of any role & the trace prefix, the state under \cref{def:sufficiency} & $c_j$ continuations of every downstream role to the terminal reward, for each of $\fanout$ sampled alternatives & the continuation-weighted mean of the other alternatives at the same decision point \\
\bottomrule
\end{tabular}
\end{table}

\end{document}